\documentclass[11pt]{article}

\usepackage{mathtools,float, keyval}  
\usepackage{commath}              
\usepackage{amsmath,amsfonts}

\usepackage{latexsym}
\usepackage[margin=2cm]{geometry}
\usepackage{setspace}
\usepackage{caption}
\usepackage{multirow}
\usepackage{subcaption}
\usepackage{algorithmic}
\usepackage{algorithm}
\usepackage{lastpage}
\usepackage{float}
\usepackage{wrapfig}
\usepackage{adjustbox}
\usepackage{hyperref}
\usepackage{xcolor}
\usepackage{xparse}
\usepackage{lineno}
\usepackage{subcaption}
\usepackage{cleveref}
\usepackage{placeins}
\usepackage{textcomp}
\usepackage{xparse}
\usepackage{hyperref}                 
\usepackage{authblk}       			  
\usepackage{placeins}

\newtheorem{theorem}{Theorem}[section]

\newtheorem{corollary}[theorem]{Corollary}
\newtheorem{proposition}[theorem]{Proposition}

\newcommand{\rvc}[2]{
\textcolor{red}{#1}%
    \IfNoValueF {#2} {\protect\footnote{\textcolor{red}{#2}}%
    }%
}

\begin{document}

\title{A Perturbation-Based Kernel Approximation Framework}

\author[1]{Roy Mitz\thanks{roymitz@mail.tau.ac.il} }
\author[1]{Yoel Shkolnisky\thanks{yoelsh@post.tau.ac.il}}
\affil[1]{\footnotesize{School of Mathematical sciences, Tel-Aviv University, Tel-Aviv, Israel}}
\date{}

\maketitle

\begin{abstract}%

Kernel methods are powerful tools in various data analysis tasks. Yet, in many cases, their time and space complexity render them impractical for large datasets. Various kernel approximation methods were proposed to overcome this issue, with the most prominent method being the Nystr{\"o}m method. In this paper, we derive a perturbation-based kernel approximation framework building upon results from classical perturbation theory. We provide an  error analysis for this framework, and prove that in fact, it generalizes the Nystr{\"o}m method and several of its variants. Furthermore, we show that our framework gives rise to new kernel approximation schemes, that can be tuned to take advantage of the structure of the approximated kernel matrix. We support our theoretical results numerically and demonstrate the advantages of our approximation framework on both synthetic and real-world data.
\end{abstract}


\section{Introduction}

In the last decades, kernel methods became widely used in various data analysis tasks. Two notable examples that are widely used for machine learning are  kernel SVM~\cite{cortes1995support} and kernel ridge regression~\cite{saunders1998ridge}. Another common application of kernel methods is non-linear dimensionality reduction, where the lower-dimensional embedding of the data is derived from the eigenvectors of some data-dependent kernel matrix. Examples of such dimensionality reduction algorithms include Laplacian eigenmaps~\cite{belkin2003laplacian}, LLE~\cite{roweis2000nonlinear}, Isomap~\cite{balasubramanian2002isomap}, MDS~\cite{buja2008data}, Spectral clustering~\cite{shi2000normalized, ng2002spectral}, and more. In all of the aforementioned kernel-based methods, the dimension of the kernel matrix grows linearly with the number of data points~$n$. This makes kernel-related algorithms impractical for large $n$, both in terms of memory and runtime. To overcome this issue, several methods for kernel approximation were proposed.

The most common form of kernel approximation is low-rank kernel approximation, where we seek to find a low-rank representation of the kernel matrix. It is known that the best rank-$m$ approximation of a matrix in the spectral and Frobenius norms is given by truncating its SVD decomposition. However, since the dimension of the kernel matrix grows with the number of data points, the computation of the full or even the partial eigendecomposition of a large kernel matrix is impractical due to its runtime and space requirements. For example, algorithms for partial eigendecomposition such as the Lanczos algorithm and some variants of SVD require $O(n^2m)$ floating point operations, where $n$ is the dimension of the matrix (number of data points) and $m$ is the number of eigenvectors calculated. Randomized algorithms~\cite{halko2011algorithm, halko2011finding} use random projections of the data to reduce the time complexity of the decomposition to $O(n^2 \log m)$, which is still impractical for large~$n$. Moreover, all eigendecomposition algorithms require to store the $n \times n$ kernel matrix either in RAM or on disk. 

Since the exact calculation of the low-rank decomposition of a kernel matrix is unfeasible for large datasets, several approximated methods were proposed. The most prominent low-rank kernel approximation method is the Nystr{\"o}m method~\cite{williams2001using}, that will be described in detail in the next section. Some variants of the Nystr{\"o}m method were proposed in literature: the Randomized SVD Nystr{\"o}m method~\cite{li2014large}, which uses efficient randomized eigensolvers that enable it to use more samples in order to improve performance; the ensemble Nystr{\"o}m method~\cite{kumar2009ensemble}, which averages several Nystr{\"o}m approximations in order to improve accuracy (we note that this is in fact not a low-rank approximation); the spectral shifted Nystr{\"o}m method~\cite{wang2014improving}, which provides superior performance in cases where the spectrum of the matrix decays slowly; and the modified Nystr{\"o}m method~\cite{wang2013improving}. The latter usually outperforms the classical Nystr{\"o}m method, but is impractical for large datasets due to its time and memory requirements. To alleviate these requirements, \cite{wang2014efficient} suggested a faster version of the modified Nystr{\"o}m method.

When the approximated kernel is not low-rank, the kernel approximation methods mentioned above may result in a poor approximation. More recently, works that use the structure of the kernel matrix were proposed. For example, the MEKA algorithm~\cite{si2017memory} provides superior kernel approximation for kernels that admit a block-diagonal structure, and are not necessarily low-rank.

A problem related to low-rank approximation that is relevant to this paper is updating a known eigendecomposition of a matrix following a ``small'' perturbation, without calculating the entire decomposition from scratch. Classical perturbation results~\cite{stewart1990matrix, byron2012mathematics} exist for a general symmetric perturbation, and will be described in detail in the next section. Other related works consider perturbations that have some structure; see for example,~\cite{bunch1978rank,mitz2019symmetric} for the case where the perturbation is of rank one, and~\cite{oh2018multiple,brand2006fast} for a general low-rank perturbation. Other approaches of updating a given  eigendecomposition include restarting the power method~\cite{langville2006updating} or the inverse iteration algorithm~\cite{trefethen1997numerical}, both require applying the updated matrix several times until convergence, which may be expensive if the matrix is large. 

The contribution of the current paper is threefold. First, we derive eigendecomposition perturbation formulas accompanied by error bounds for matrices that only part of their spectrum is known. Second, we use these perturbation formulas to derive a new framework for kernel approximation. Unlike some of the existing methods, we present explicit error bounds for the individual approximated eigenvectors rather than only to the kernel approximation. Third, we prove that the Nystr{\"o}m method and its variants are in fact special cases of our framework. This reveals the essence behind existing Nystr{\"o}m methods, allows to analyze their accuracy, and provides means to derive new Nystr{\"o}m-type approximations that take advantage of the structure of the kernel matrix. 

The rest of this paper is organized as follows. In Section~\ref{sec:prem}, we describe classical perturbation results, along with the Nystr{\"o}m method and some of its variants. In Section~\ref{sec:trunc_pert}, we extend the classical perturbation formulas to the case where only part of the spectrum of the perturbed matrix is known, and derive their error bounds. In Section~\ref{sec:the_extension_framework}, we use these formulas to develop a perturbation-based kernel approximation framework. In Section~\ref{sec:nys_is_pert}, we prove that our kernel approximation framework generalizes the Nystr{\"o}m method. In Section~\ref{sec:applications}, we suggest several types of kernesl approximations based on our framework, and prove that some of them are related to variants of the Nystr{\"o}m method. In Section~\ref{sec:numerical}, we provide numerical results to support our theory and show the advantages of our kernel approximation framework. In Section~\ref{sec:summary}, we provide some concluding remarks and discuss future research.

\section{Preliminaries} \label{sec:prem}

In this section, we describe two methods for approximating the eigendecomposition of a matrix that are relevant to our work. We first describe the Nystr{\"o}m method in Section~\ref{sec:preliminaries_nystrom}, and then describe the perturbation method in Section~\ref{sec:pert_eig}.

\subsection{The Nystr{\"o}m method and its variants}\label{sec:preliminaries_nystrom}

In this section, we describe in detail some of the Nystr{\"o}m-type methods discussed in the Introduction. We begin with describing the classical Nystr{\"o}m method, continue with some of its variants that are particularly relevant to our work, and finish by discussing some results regarding the error bounds of the Nystr{\"o}m method.

\subsubsection{The classical Nystr{\"o}m method}\label{sec:preliminaries_nystrom_base}

Let $K \in \mathbb{R}^{n \times n}$ be a symmetric positive-definite matrix. We wish to find the~$k$ leading eigenpairs $\{(\lambda_i,u_i)\}_{i=1}^k$ of~$K$. The Nystr{\"o}m method~\cite{sun2015review,williams2001using} finds an approximation $\{(\widetilde{\lambda}_i,\widetilde{u}_i)\}_{i=1}^k$ to these eigenpairs as follows. First, we select randomly $k$~columns of~$K$ (typically uniformly without replacement). We assume, without loss of generality, that the columns and rows of~$K$ are rearranged so that the first~$k$ columns of~$K$ are sampled. Denote by $K'$ the $k \times k$ upper-left submatrix of $K$ and by~$C$ the $n \times k$ matrix consisting of the first~$k$ columns of~$K$. Then, we calculate the~$k$ eigenpairs of~$K'$ and denote them by $\{(\lambda_i', u_i' )\}_{i=1}^{k}$. Finally, the Nystr{\"o}m method approximates the $k$ leading eigenvectors of $K$ by
\begin{equation} \label{eq:nystrom}
\widetilde{u}_i = \sqrt{\frac{k}{n}} \frac{1}{\lambda_i'}Cu_i', \quad  i=1,\ldots,k. 
\end{equation}
Moreover, the $k$ leading eigenvalues of $K$ are approximated by 
\begin{equation} \label{eq:nystrom_vals}
\widetilde{\lambda}_i = \frac{n}{k} \lambda_i', \quad i=1,\ldots,k.
\end{equation}
Finally, the Nystr{\"o}m approximation of $K$ is
\begin{equation} \label{eqn:nys_k_app}
\widetilde{K}_{\text{nys}} = \sum_{i=1}^{k} \widetilde{\lambda}_i \widetilde{u}_i^{T} \widetilde{u}_i . 
\end{equation}
The runtime complexity of the Nystr{\"o}m method (Formula~\eqref{eq:nystrom}) is $O(nk^2 + k^3)$.

The Nystr{\"o}m method requires sampling a representative subset of the data and hence many methods for sampling this subset were proposed, see~\cite{kumar2012sampling,sun2015review} and the references therein. Our results derived below are independent of the methodology used to obtain the subset of samples, and hence we do not discuss this issue in detail.

\subsubsection{Variants of the Nystr{\"o}m method}\label{sec:preliminaries_nystrom_var}

One generalization of the Nystr{\"o}m method is the Randomized SVD Nystr{\"o}m method~\cite{li2014large}. It introduces a parameter $l \geq k$ and chooses $K'$ to be the $l \times l$ top-left submatrix of $K$. Then, the $k$ leading eigenpairs of $K'$ are calculated via some efficient randomized SVD method, followed by the use of~\eqref{eq:nystrom} and~\eqref{eq:nystrom_vals} for the approximation. This form of approximation is a generalization of the Nystr{\"o}m method, since choosing $l = k$ is equivalent to the Nystr{\"o}m method. On the other hand, if we choose $l = n$, we get the exact eigenvectors of $K$. Intuitively, the larger $l$ is, the better the approximation is likely to be, at the cost of a greater computational complexity. The runtime complexity of this method is $O(nk^2 + lk^2)$. 

Since the Nystr{\"o}m approximation~\eqref{eqn:nys_k_app} of the kernel matrix $K$ is a low-rank approximation, it may provide poor results when $K$ is not low-rank. This might occur, for example, when the spectrum of $K$ decays slowly. A possible approach to overcome this problem is the spectrum shifted Nystr{\"o}m method~\cite{wang2014improving}. This method essentially applies the classical Nystr{\"o}m method on a shifted kernel matrix, that is, applies the Nystr{\"o}m method on 
\begin{equation} \label{eq:shift}
K_{\text{shift}} = K - \mu I,
\end{equation} 
for some $\mu \geq 0$. The updated eigenvalues~\eqref{eq:nystrom_vals} are then shifted-back by $\mu$. If we denote by~$\widetilde{K}_{\text{shift}}$ the kernel approximation obtained by using the shifted Nystr{\"o}m method, it is shown in~\cite{wang2014improving} that
\begin{equation*}
    \norm{K - \widetilde{K}_{\text{shift}}}_F \leq \norm{K - \widetilde{K}_{\text{nys}}}_F .
\end{equation*}

Alternatively, the kernel $K$ may admit a block-diagonal structure. As demonstrated in~\cite{si2017memory}, this may happen for some kernel functions when the data consist of several clusters. In this case, the MEKA algorithm~\cite{si2017memory} essentially performs a Nystr{\"o}m approximation on each cluster of the data. Each such approximation corresponds to a block on the diagonal of the kernel matrix, and the resulting approximation is block-diagonal.

An approach related to the MEKA algorithm for improving the Nystr{\"o}m approximation~\eqref{eqn:nys_k_app} is the ensemble Nystr{\"o}m method~\cite{kumar2009ensemble}. The idea behind this method is to perform $q$ independent Nystr{\"o}m kernel approximations on random subsets of the data, and then average them. Formally, given $q$ independent Nystr{\"o}m approximations $\{ \widetilde{K}_i \}_{i=1}^q$, the ensemble Nystr{\"o}m approximation is given by
\begin{equation*}
\widetilde{K}_{\text{ens}} = \sum_{i=1}^q \mu_i \widetilde{K}_i ,
\end{equation*}
for some weights $\{\mu_i\}_{i=1}^q$. It is suggested in~\cite{kumar2009ensemble} to use $\mu_i = \frac{1}{q}$ for $1 \leq i \leq q$. Better error bounds for this method compared to the classical Nystr{\"o}m method are proven in~\cite{kumar2009ensemble}. 

The difference between the ensemble Nystr{\"o}m method and the MEKA algorithm is that in the former, the individual Nystr{\"o}m approximations are chosen at random rather than by clusters, and the resulting approximation is their average rather than their concatenation in a block-diagonal matrix. We note that in both the ensemble Nystr{\"o}m method and the MEKA algorithm, the resulting approximation is not low-rank, and is generally of a rank greater than $k$.

\subsubsection{Error bounds for the Nystr{\"o}m method }\label{sec:preliminaries_nystrom_err}

The error bounds of the Nystr{\"o}m method and its variants have been of great interest. For a comprehensive review we refer the readers to~\cite{sun2015review},~\cite{wang2013improving} and \cite{gittens2013revisiting}. If we denote by $K_m$ the best rank-$m$ approximation of a kernel  matrix~$K$, then all error bounds obtained in the literature are of the forms
\begin{equation*}
    \norm{K - \widetilde{K}_{\text{nys}}} \leq \alpha \norm{K - K_m} \quad \text{or} \quad \norm{K - \widetilde{K}_{\text{nys}}} \leq \norm{K - K_m} + \beta,
\end{equation*}
for some $\alpha,\beta \geq 0$, where $\norm{\cdot}$ is usually the Frobenius norm or the $2$-norm. To the best of our knowledge, no error bounds were obtained for the individual Nystr{\"o}m-approximated eigenvectors, though such bounds may be of interest when using the Nystr{\"o}m method as  part of a dimensionality reduction algorithm. 

\subsection{Perturbation of eigenvalues and eigenvectors} \label{sec:pert_eig}

Let $A' \in \mathbb{R}^{n \times n}$ be a real symmetric positive-definite matrix with distinct eigenvalues $\{t_i\}_{i=1}^n$ and their corresponding orthonormal eigenvectors $\{v_i\}_{i=1}^n$. Assume that $t_1 > t_2 > \cdots > t_n$. Let $E \in \mathbb{R}^{n \times n}$ a real symmetric matrix. Consider a perturbation $A$ of $A'$ given by $A=A'+E$, with the eigenpairs of~$A$ denoted by $\{(s_i,w_i)\}_{i=1}^n$, so that $s_1 > s_2 > \cdots > s_m$. We wish to find an approximation $\{(\widetilde{s}_i,\widetilde{w}_i)\}_{i=1}^n$ to the eigenpairs of $A$. The classical perturbation solution to this problem~\cite{stewart1990matrix, byron2012mathematics} is as follows. The approximated eigenvectors of $A$ are given by
\begin{equation} \label{eqn:pert_org_vecs}
\widetilde{w}_{i} = v_{i} + \sum_{k=1, k \neq i}^{n} \frac{ \langle Ev_i,v_k \rangle }{t_i-t_k}v_k + O(\norm{E}^2_2), \quad 1 \leq i \leq n,
\end{equation}
where $\langle \cdot,\cdot \rangle$ is the standard dot product between vectors, and the approximated eigenvalues of $A$ are given by
\begin{equation} \label{eqn:pert_org_vals}
\widetilde{s}_i = t_i + v_i^TEv_i + O(\norm{E}^2_2), \quad 1 \leq i \leq n .
\end{equation}
We note that the eigenvalues update formula~\eqref{eqn:pert_org_vals} depends only on the eigenvalue we wish to update and its corresponding eigenvector, whereas the eigenvectors update formula~\eqref{eqn:pert_org_vecs} depends on all eigenvalues and eigenvectors of $A'$.

There exist perturbation results for matrices with non-simple eigenvalues~\cite{byron2012mathematics}. However, since non-simple eigenvalues are highly unlikely in data-dependent matrices, we do not discuss this case and leave it for a future work.

\section{Truncating the perturbation formulas} \label{sec:trunc_pert}

In this section, we consider a variant of the problem presented in Section~\ref{sec:pert_eig}, in which only the $m$ leading eigenpairs $\{(t_i,v_i)\}_{i=1}^m$ of the unperturbed matrix $A'$ are known, and we wish to approximate the $m$ leading eigenpairs of $A$. To this end, we introduce a parameter $\mu \in \mathbb{R}$ whose purpose is to approximate the unknown eigenvalues $\{t_i\}_{i={m+1}}^n$ of $A'$. We denote by $V^{(m)}$ the $n \times m$ matrix consisting of the $m$ leading eigenvectors of $A'$, and define
\begin{equation} \label{eqn:r}
    r_i = \Big( I - V^{(m)}V^{(m)T} \Big) Ev_i ,
\end{equation}
for $1 \leq i \leq m$.

We derive two approximation formulas based on the classical perturbation formula~\eqref{eqn:pert_org_vecs}. These two approximation formulas differ in their order of approximation and in their computational complexity. The first formula, which we refer to as the first-order truncated perturbation formula, provides a first-order approximation to the eigenvectors of $A$, whereas the second formula, which we refer to as the second-order truncated perturbation formula, provides a second-order approximation to the eigenvectors of $A$. We describe these approximations in the following propositions.

\begin{proposition}[First-order approximation] \label{prop:pert_partial_1} 

Let $1 \leq i 
\leq m$. Let $\mu \in \mathbb{R}$ a parameter. Using the notation of Section~\ref{sec:pert_eig}, the first-order approximation to $w_i$ is given by
\begin{equation} \label{eq:pert_expansion}
\widetilde{w}_{i}^{(1)} = v_{i} + \sum_{k=1, k \neq i}^{m} \frac{ \langle Ev_i,v_k \rangle }{t_i-t_k}v_k  +  \frac{1}{t_i- \mu} r_i ,
\end{equation}
with an error satisfying
\begin{equation} \label{eqn:error_trunc_1}
    \norm{w_i - \widetilde{w}_i^{(1)}}_2 \leq \frac{ \sum_{k=m+1}^n \abs{t_k - \mu}  }{\abs{t_i - t_{m+1}}\abs{t_i - \mu}} \norm{E}_2 + O \big( \norm{E}^2_2 \big).
\end{equation}

\end{proposition}

\begin{proposition}[Second-order approximation] \label{prop:pert_partial_2}
Let $1 \leq i 
\leq m$. Let $\mu \in \mathbb{R}$ a parameter. Using the notation of Section~\ref{sec:pert_eig}, the second-order approximation to $w_i$ is given by
\begin{equation} \label{eq:pert_expansion2}
\widetilde{w}_{i}^{(2)} = v_{i} + \sum_{k=1, k \neq i}^{m} \frac{ \langle Ev_i,v_k \rangle }{t_i-t_k}v_k  +  \frac{1}{t_i- \mu} r_i -  \frac{\mu}{(t_i- \mu)^2} r_i +  \frac{1}{(t_i- \mu)^2} A'r_i,
\end{equation}
with an error satisfying
\begin{equation} \label{eqn:error_trunc_2}
    \norm{w_i - \widetilde{w}_i^{(2)}}_2 \leq \frac{ \sum_{k=m+1}^n \abs{t_k - \mu}^2  }{\abs{t_i - t_{m+1}}\abs{t_i - \mu}^2} \norm{E}_2 + O \big( \norm{E}^2_2 \big).
\end{equation}

\end{proposition}

The proofs of Proposition~\ref{prop:pert_partial_1} and Proposition~\ref{prop:pert_partial_2} are given in Appendix~\ref{app1} and Appendix~\ref{app2}, respectively. We discuss the runtime and memory requirements of formulas~\eqref{eq:pert_expansion} and~\eqref{eq:pert_expansion2} in Appendix~\ref{app:runtime}.

The update formulas~\eqref{eq:pert_expansion} and~\eqref{eq:pert_expansion2} depend on a parameter $\mu$, whose choice is discussed in~\cite{mitz2019symmetric}. To be concrete, if $A'$ is known to be low-rank, we use $\mu = 0$. When $A'$ is not low-rank, and especially if its spectrum is known to decay slowly, we follow~\cite{mitz2019symmetric} and suggest to use
\begin{equation} \label{eqn:mu_mean}
     \mu_{\text{mean}} = \frac{\text{trace}(A') - \sum_{i=1}^{m}t_i}{n-m},
\end{equation}
which is the mean of the unknown eigenvalues of the perturbed matrix $A'$. $\mu_{\text{mean}}$ can be easily computed since the trace of $A'$ and its $m$ leading eigenvalues are known.

We conclude this section by proving that under a certain assumption on $A'$, the first-order truncated perturbation formula~\eqref{eq:pert_expansion} and the second-order truncated perturbation formula~\eqref{eq:pert_expansion2} coincide. Furthermore, in this case, the $O \big( \norm{E}_2 \big) $ term in the error bound of both approximations cancels out, as stated by the following proposition.

\begin{proposition} \label{prop:order1_is_order2}
Let $\delta \geq 0$ and assume that $A'$ can be written in the form of a low-rank matrix plus a spectrum shift, that is $A' = V^{(m)}TV^{(m)T} + \delta I$ for some diagonal matrix $T \in \mathbb{R}^{m \times m}$. Then, for $\mu = \delta$, the first-order truncated perturbation formula~\eqref{eq:pert_expansion} and the second-order truncated perturbation formula~\eqref{eq:pert_expansion2} coincide, that is
\begin{equation*}
    \widetilde{w}_{i}^{(1)} = \widetilde{w}_{i}^{(2)},
\end{equation*}
and the approximation errors satisfy
\begin{equation*}
    \norm{w_i - \widetilde{w}_{i}^{(1)}}_2 = \norm{w_i - \widetilde{w}_{i}^{(2)}}_2 = O \big( \norm{E}^2_2 \big),
\end{equation*}
for all $1 \leq i \leq m$.
\end{proposition}
The proof of Proposition~\ref{prop:order1_is_order2} is given in Appendix~\ref{app3}.
\begin{corollary} \label{col:first_is_second}
If $A'$ is of rank $m$, and $\mu = 0$ in~\eqref{eq:pert_expansion} and~\eqref{eq:pert_expansion2}, then the first-order and second-order truncated perturbation formulas give rise to the same approximation. The error of this approximation is $ O \big( \norm{E}_2^2 \big)$.
\end{corollary}

\section{Perturbation-based kernel approximation framework}  \label{sec:the_extension_framework}

In this section, we derive our perturbation-based kernel approximation framework based on Proposition~\ref{prop:pert_partial_1}. Let $K \in \mathbb{R}^{n \times n}$ be a symmetric positive-definite matrix with $m$ distinct leading eigenpairs that are denoted by $\{ (\lambda_i, u_i ) \}_{i=1}^m$, ordered in descending order. Let $K^s\in \mathbb{R}^{n \times n}$ be a symmetric matrix consisting of any subset of entries of $K$, with the rest of its entries being~$0$, as illustrated in Figure~\ref{fig:K_star}. Our kernel approximation framework approximates the eigenvectors of $K$ using the eigenvectors of any such $K^s$, as follows. 

Let $\{ (\lambda_i^s, u_i^s ) \}_{i=1}^m$ be the leading eigenpairs of $K^s$ ordered in a descending order, and assume that $\lambda_i^s$ are distinct. Let $U^{s(m)} \in \mathbb{R}^{n \times m}$ be the matrix whose columns are the~$m$ eigenvectors of $K^s$ corresponding to the~$m$ largest eigenvalues of $K^s$, and let $\mu \geq 0$ be a parameter. Let $1 \leq i \leq m$. By the first-order approximation in Proposition~\ref{prop:pert_partial_1}, the eigenvector $u_i$ is approximated by
\begin{equation} \label{eqn:pert_extension}
\widetilde{u}_{i} = u^s_{i} + \sum_{k=1, k \neq i}^{m} \frac{ \langle (K - K^s)u^s_i,u^s_k  \rangle}{\lambda^s_i-\lambda^s_k}u^s_k  +  \frac{1}{\lambda^s_i- \mu} \big( I_m - U^{s(m)}U^{s(m)T} \big) (K - K^s)u^s_i  ,
\end{equation}
with an error satisfying
\begin{equation} \label{eqn:ext_error}
\norm{u_i - \widetilde{u}_i}_2 \leq \frac{ \sum_{k=m+1}^n \abs{\lambda^s_k - \mu}  }{\abs{\lambda^s_i - \lambda^s_m}\abs{\lambda^s_i - \mu}} \norm{K - K^s}_2 + O \big( \norm{K - K^s}^2_2 \big).
\end{equation}
Furthermore, by~\eqref{eqn:pert_org_vals}, the eigenvalue $\lambda_i$ is approximated by
\begin{equation} \label{eqn:pert_ext_vals}
\widetilde{\lambda}_i = \lambda_i^s + u_i^{sT}(K - K^s)u^s_i,
\end{equation}
with an error of magnitude $\abs{\lambda_i - \widetilde{\lambda}_i } = O (\norm{K - K^s}^2_2)$.

We refer to~\eqref{eqn:pert_extension} and~\eqref{eqn:pert_ext_vals} as our perturbation approximation for the eigenvectors and eigenvalues, respectively. Finally, we define our perturbation kernel approximation by
\begin{equation} \label{eq:k_pert}
    \widetilde{K}_{\text{pert}} = \sum_{i=1}^m \widetilde{\lambda}_i \widetilde{u}_i^T \widetilde{u}_i.
\end{equation}

Note that this framework is quite general, and can be applied to any symmetric sub-matrix of~$K$. We will propose and discuss several choices for $K^s$ in Section~\ref{sec:applications}.

A kernel approximation framework analogous to~\eqref{eqn:pert_extension} that is based on the second-order approximation in Proposition~\ref{prop:pert_partial_2} can also be obtained in a similar way, but is typically impractical for large~$n$ due to its time and space requirements.

\begin{figure}
  \centering
  \begin{subfigure}[b]{0.3\linewidth}
    \includegraphics[width=\linewidth]{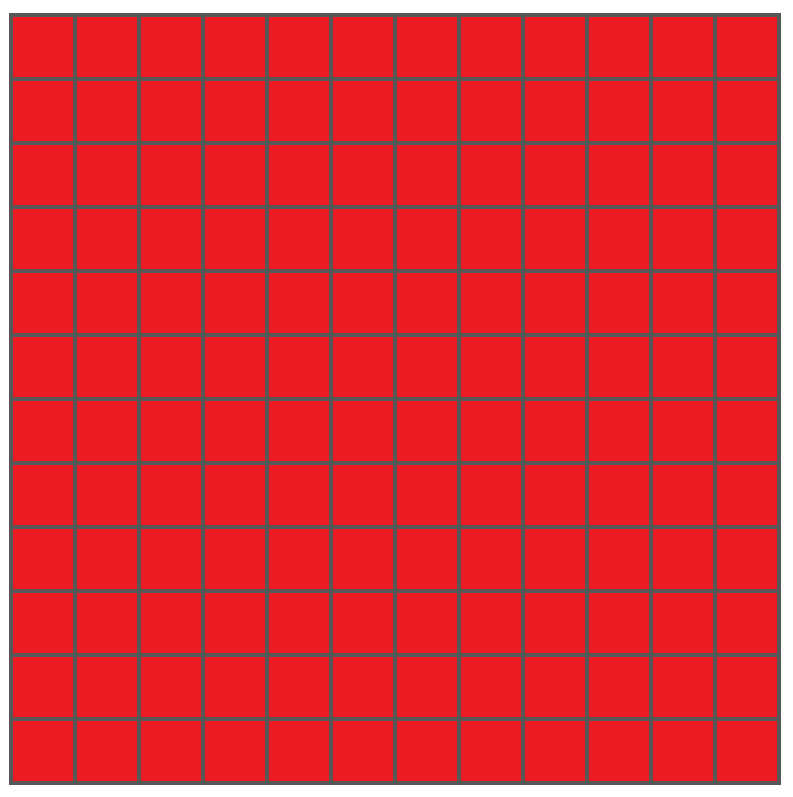}
    \caption{$K$.}
  \end{subfigure}
  \begin{subfigure}[b]{0.3\linewidth}
    \includegraphics[width=\linewidth]{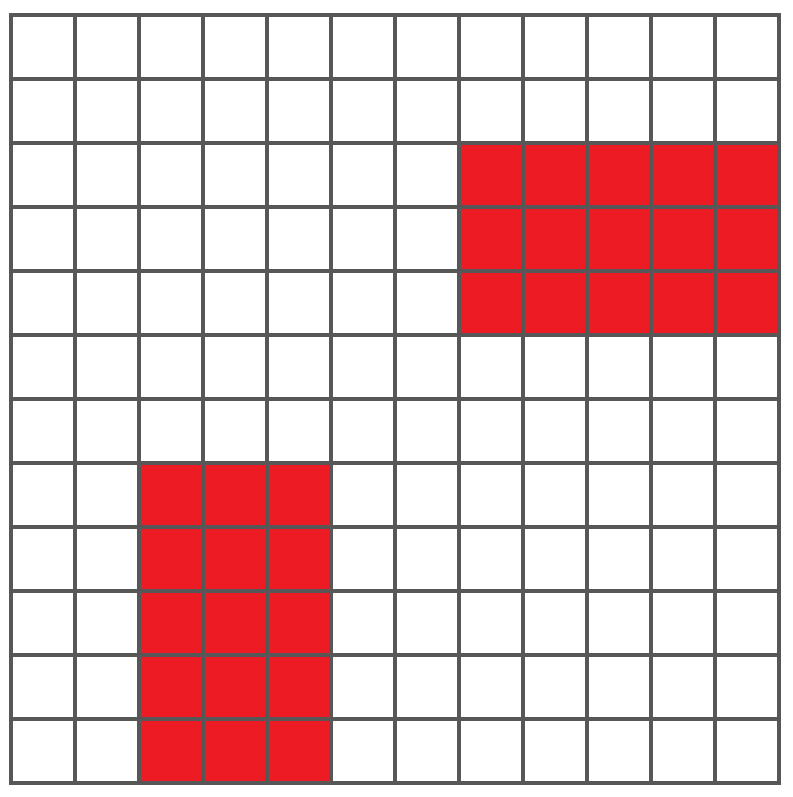}
    \caption{Possible $K^s$.}
  \end{subfigure}
  \begin{subfigure}[b]{0.3\linewidth}
    \includegraphics[width=\linewidth]{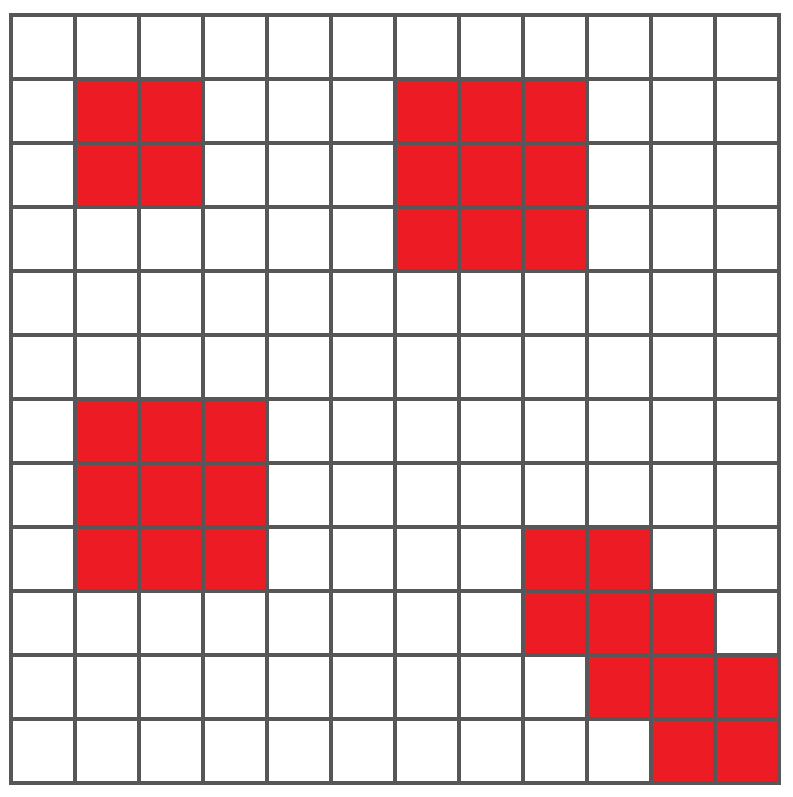}
    \caption{Another possible $K^s$.}
  \end{subfigure}
  \caption{Illustration of the submatrix $K^s$. Blank entries indicate $0$.}
  \label{fig:K_star}
\end{figure}

\section{Equivalence with the Nystr{\"o}m method} \label{sec:nys_is_pert}

In this section, we prove that our perturbation-based kernel approximation framework~\eqref{eqn:pert_extension} is in fact a generalization of the Nystr{\"o}m method described in Section~\ref{sec:preliminaries_nystrom}, by showing that the Nystr{\"o}m method arises from our kernel approximation framework by a specific choice of $K^s$.

Let $K \in \mathbb{R}^{n \times n}$ be a kernel matrix, and let $m < n$. Assume without loss of generality, that we apply the classical Nystrom method, given in~\eqref{eq:nystrom} and~\eqref{eq:nystrom_vals}, to the $m\times m$ upper-left submatrix of $K$, and denote by $\{ (\hat{\lambda}_i, \hat{u}_i) \}_{i=1}^m$ the resulting approximate eigenpairs of $K$. The following proposition states that for a specific choice of the matrix~$K^s$, the perturbation approximation $\widetilde{K}_{\text{pert}}$ of~\eqref{eq:k_pert} is exactly the classical Nystr{\"o}m approximation $\widetilde{K}_{\text{nys}}$ of~\eqref{eqn:nys_k_app}.

\begin{proposition} \label{prop:nys_is_pert}
Using the above notation, let $K^s$ be the $n\times n$ matrix whose top left $m \times m$ submatrix is the top left $m \times m$ submatrix of $K$, and the rest of its entries are 0. Denote by $\{ (\lambda_i^s, u_i^s) \}_{i=1}^{m}$ the eigenpairs of $K^s$. Set $\mu$ in~\eqref{eqn:pert_extension} to be~$0$, and denote by $\{ (\widetilde{\lambda}_i, \widetilde{u}_i) \}_{i=1}^m$ the perturbation approximation of the eigenpairs $\{ (\lambda_i^s, u_i^s) \}_{i=1}^{m}$ of $K^s$ to the eigenpairs of $K$. Denote by $K'$ the top left $m \times m$ submatrix of $K$ and by $\{ (\lambda_i',u_i') \}_{i=1}^m$ its eigenpairs. Denote by $\{ (\hat{\lambda}_i,\hat{u}_i) \}_{i=1}^m$ the Nystr{\"o}m approximation of $\{ (\lambda_i',u_i') \}_{i=1}^m$ (see~\eqref{eq:nystrom} and~\eqref{eq:nystrom_vals}). Then,
\begin{equation}\label{eqn:ev equivalence}
     \hat{u}_i = \sqrt{\frac{m}{n}} \widetilde{u}_i \quad \text{and} \quad \hat{\lambda}_i = \frac{n}{m} \widetilde{\lambda}_i
\end{equation}
for all $1 \leq i \leq m$. In particular,

\begin{equation*}
    \widetilde{K}_{\text{nys}} = \widetilde{K}_{\text{pert}}.
\end{equation*}
\end{proposition}

The proof of Proposition~\ref{prop:nys_is_pert} is given in Appendix~\ref{app4}.

Formulating the Nystr{\"o}m method as a perturbation-based approximation using Proposition~\ref{prop:nys_is_pert} enables us to provide an error bound for the Nystr{\"o}m method based on Propositions~\ref{prop:pert_partial_1} and~\ref{prop:pert_partial_2}. Contrary to previous works that only provide error bounds for the approximated kernel resulting from the Nystr{\"o}m method~\eqref{eqn:nys_k_app}, our error bound is for the individual approximated eigenvectors, as stated in the following proposition.

\begin{proposition} [Error bound for the Nystr{\"o}m method]  \label{prop:pert_error_for_nys}
Using the above notation, the error induced by the Nystr{\"o}m method satisfies
\begin{equation*} 
\norm{u_i - \hat{u}_i}_2 =  O \big( \norm{K - K^s}^2_2 \big), \quad 1 \leq i \leq m .
\end{equation*}
\end{proposition}

The proof follows directly from the equivalence stated in Proposition~\ref{prop:nys_is_pert} by noting that in the Nystr{\"o}m method setting, the requirements of Corollary~\ref{col:first_is_second} hold.

\section{New kernel approximation schemes based on the perturbation framework} \label{sec:applications}

The perturbation-based kernel approximation framework derived in Section~\ref{sec:the_extension_framework} is very flexible, and allows for various approximations that depend on the choice of the matrix $K^s$. There are two main considerations in choosing $K^s$. First, since calculating the eigendecomposition of $K^s$ is the most expensive part of the perturbation-based kernel approximation, we wish to choose a matrix $K^s$ whose eigendecomposition is ``easy" to compute. This will make our framework computationally attractive. Second, we would like to take advantage of the flexibility of our framework and choose a matrix $K^s$ that ``captures" as much of a given kernel $K$ as possible (that is, to minimize the $\norm{K - K_s}_2$ term in~\eqref{eqn:pert_extension}). This will allow for better approximation results compared to classical Nystr{\"o}m-type methods.

In this section, we propose several approximation schemes corresponding to different choices of $K^s$. Furthermore, we prove that several of the Nyst{\"o}om method variants described in Section~\ref{sec:prem} also arise from our general approximation framework for suitable choices of $K^s$. The list below is not by all means comprehensive, and users might come up with different approximation schemes that are more suitable to their problems' settings.

\subsection{\texorpdfstring{$l$}{}-block kernel approximation} \label{sec:l_ext}

We define the $l$-block kernel approximation by choosing the matrix $K^s$ to be the top left $l \times l$ submatrix of $K$ padded with zeros to size $n \times n$ (see Figure~\ref{fig:app_illus}), with $l \geq m$ being a parameter. The difference of this approach from the Nystr{\"o}m method is that while we still calculate $m$ eigenpairs, we do so on a larger block, which ``captures'' more of $K$. This comes at the price of a greater computational cost.

\begin{proposition}
Using the notation of this section, the eigenpairs calculated by the randomized SVD Nystr{\"o}m method~\cite{li2014large} and the $l$-block kernel approximation method are equal.
\end{proposition}

The proof easily follows from the definitions and Proposition~\ref{prop:nys_is_pert}.

\subsection{\texorpdfstring{$\mu$}{}-shifted kernel approximation} \label{sec:shifted_ext}

We define the $\mu$-shifted kernel approximation by choosing the matrix $K^s$ to be the top left $m \times m$ submatrix of~$K$ padded with zeros to size $n \times n$, similarly to the Nystr{\"o}m method (see Figure~\ref{fig:app_illus}). The difference from the Nystr{\"o}m method lies in the parameter $\mu$ of~\eqref{eqn:pert_extension}. In Proposition~\ref{prop:nys_is_pert}, we used the value of the parameter $\mu$ to be $\mu=0$. This is a reasonable choice when the kernel matrix~$K$ is low-rank, or when its spectrum decays fast. When that is not the case, it might be beneficial to choose a parameter $\mu$ that approximates the unknown eigenvalues of~$K$. A reasonable choice for~$\mu$ in such a case is $\mu_{mean}$ of~\eqref{eqn:mu_mean}.

We now prove that given a parameter $\mu \geq 0$, the spectral shifted Nystr{\"o}m method~\cite{wang2014improving} with parameter $\mu$ coincides with the perturbation approximation method with the same $\mu$, as detailed in the following proposition.

\begin{proposition} \label{prop:ss_nys_is_pert}
Using the notation of this section, the eigenpairs calculated by the spectrum shifted Nystr{\"o}m method~\cite{wang2014improving} and the $\mu$-shifted approximation method are equal.
\end{proposition}
The proof of Proposition~\ref{prop:ss_nys_is_pert} is given in Appendix~\ref{app5}.

\subsection{Block-diagonal kernel approximation} \label{sec:block_ext}

We define the block-diagonal kernel approximation by choosing the matrix $K^s$ to be a block diagonal matrix (see Figure~\ref{fig:app_illus}). The block sizes can be arbitrary, but for simplicity of notation, we choose $k$ blocks of an identical size $l \geq m$.  For each block, we pad the block with zeros to obtain an $n \times n$ matrix, and calculate its $m$ leading eigenpairs. We then approximate the eigenpairs of $K$ using~\eqref{eqn:pert_extension}. Denote by $\{ (\widetilde{\lambda}_i^{(j)}, \widetilde{u}_i^{(j)}) \}_{i=1}^m$ the extended eigenpairs of block~$j$, and by $\widetilde{K}_j \in \mathbb{R}^{n \times n}$ the resulting kernel approximation.
To combine the $k$ approximations $\{ \widetilde{K}_j \}_{j=1}^k$ to an approximation of $K$, we set
\begin{equation*}
    \widetilde{K} =  \frac{1}{k} \sum_{i=1}^{k}  \widetilde{K}_i .
\end{equation*}

We note that the kernel matrix approximation obtained by this method generally won't be low-rank.

\begin{proposition}
Using the notation of this section, the eigenpairs calculated by the ensemble Nystr{\"o}m method~\cite{kumar2009ensemble} and the block diagonal kernel approximation method are equal.
\end{proposition}

The proof easily follows from the definitions and Proposition~\ref{prop:nys_is_pert}.

\subsection{\texorpdfstring{$p$}{}-band kernel approximation} \label{sec:bp_ext}

We define the $p$-band kernel approximation by choosing the matrix $K^s$ to be a band matrix of width $p$ (see Figure~\ref{fig:app_illus}). This approximation may provide superior results when the kernel $K$ has most of its energy concentrated along the diagonal. Such a kernel may arise naturally for sequential data, where adjacent entries are more ``similar" to each other. In such a case, the $p$-band kernel approximation will capture most of $K$.  Existing Nystr{\"o}m-type approximations, on the other hand, are not able to do so since they are limited to block matrices.
We demonstrate the advantage of this kernel approximation method in Section~\ref{sec:num_p_band}.

\begin{figure}[ht]
  \centering
  \begin{subfigure}[t]{0.3\linewidth}
    \includegraphics[width=\linewidth]{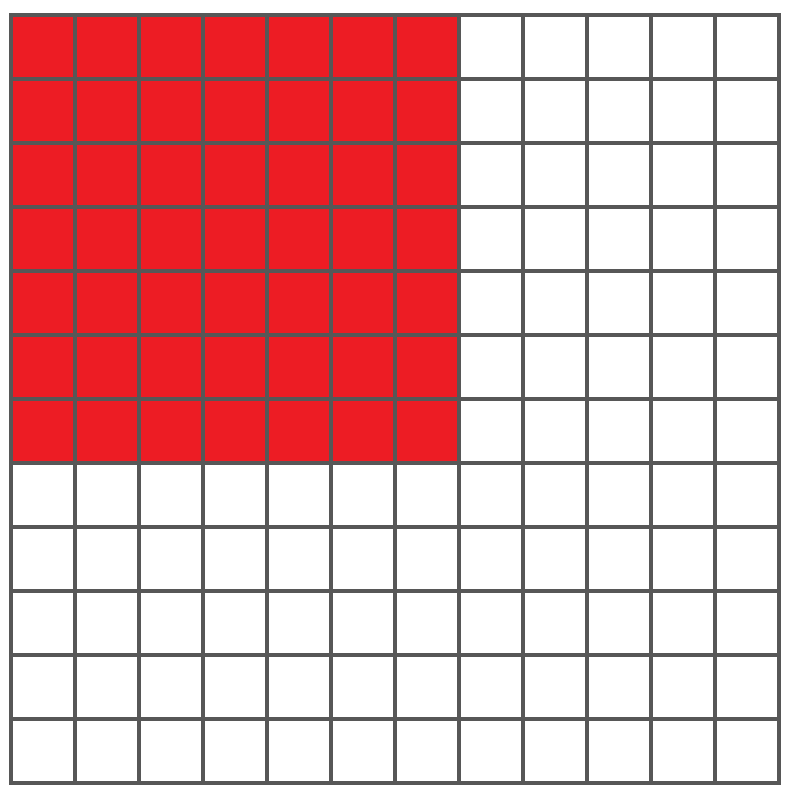}
    \caption{$K^s$ for $l$-block kernel approximation. \\ \hfill}
  \end{subfigure} \label{fig:0}
  \begin{subfigure}[t]{0.3\linewidth}
    \includegraphics[width=\linewidth]{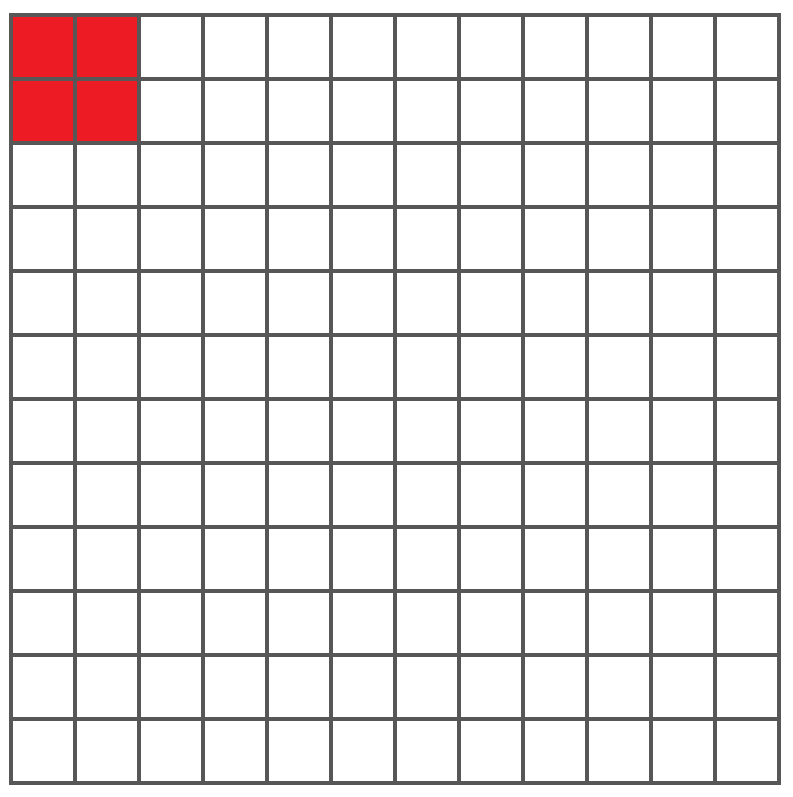}
    \caption{$K^s$ for $\mu$-shifted kernel approximation. \\ \hfill}
  \end{subfigure} \label{fig:1}
  \\
  \begin{subfigure}[t]{0.3\linewidth}
    \includegraphics[width=\linewidth]{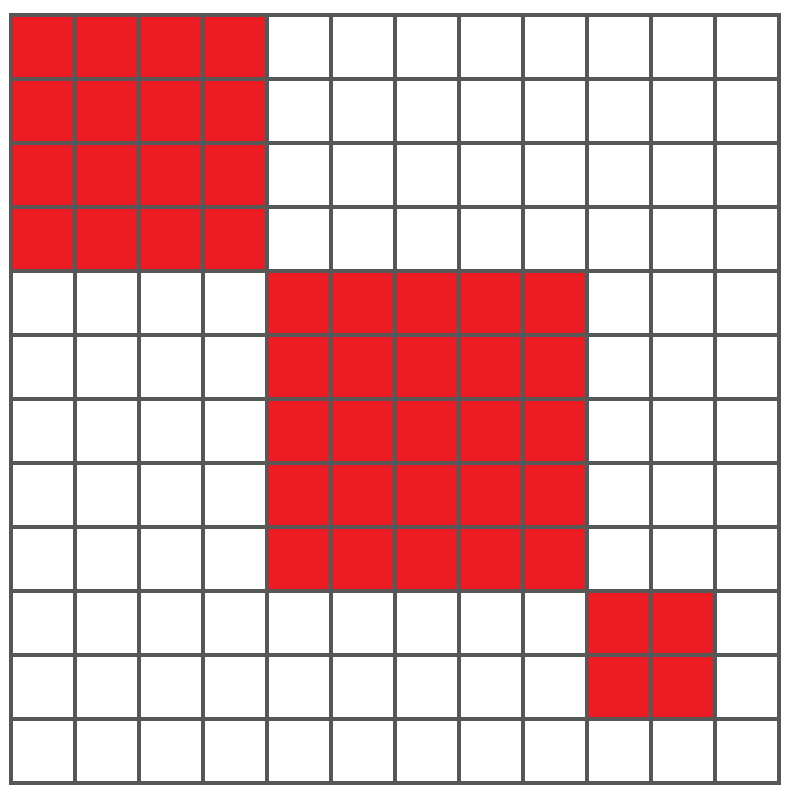}
    \caption{$K^s$ for block diagonal kernel approximation.}
  \end{subfigure}
  \begin{subfigure}[t]{0.3\linewidth}
    \includegraphics[width=\linewidth]{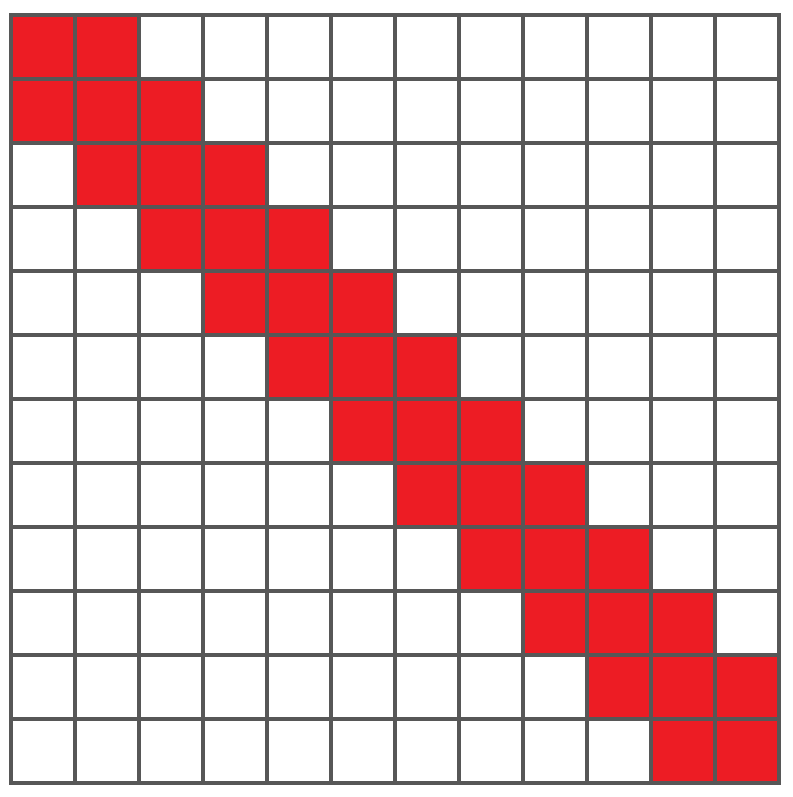}
    \caption{$K^s$ for $p$-band kernel approximation. \\ \hfill}
  \end{subfigure}
  \caption{Illustration of the submatrix $K^s$ for each of the discussed kernel approximations. Blank entries indicate $0$.}
  \label{fig:app_illus}
\end{figure}

\subsection{Sparse kernel approximation} \label{sec:sparse_ext}

We define the sparse kernel approximation by choosing the matrix $K^s$ to be some sparse submatrix of $K$, as illustrated in Figure~\ref{fig:app_sparse}. More concretely, in the sparse approximation framework, we denote by $\text{nnz}(K)$ the number of non-zero entries of $K$, and define~$K^s$ by choosing $q \cdot \text{nnz}(K)$ entries of $K$, for some $0 < q \leq 1$. While this approximation is valid for any (symmetric) subset of elements of $K$, motivated by the $\norm{E}_2$ term in the error bounds~\eqref{eqn:error_trunc_1} and~\eqref{eqn:error_trunc_2}, we suggest choosing the $q \cdot \text{nnz}(K)$ largest entries of $K$ in their absolute value. We demonstrate the advantage of this kernel approximation method in Section~\ref{sec:num_sparse}.

\begin{figure}
  \centering
  \begin{subfigure}[b]{0.3\linewidth}
    \includegraphics[width=\linewidth]{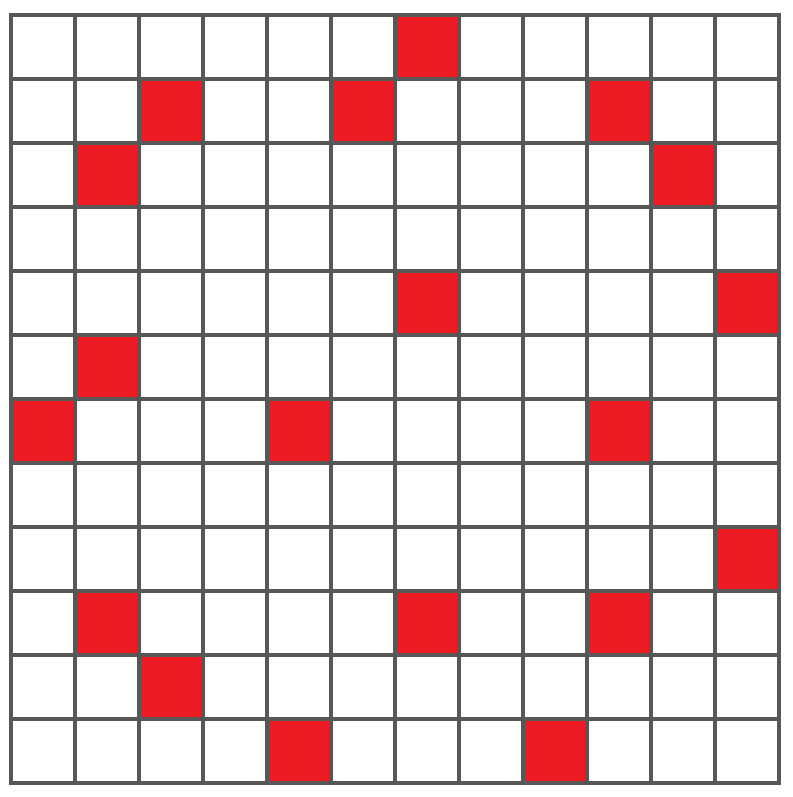}
    \caption{Sparse $K$.}
  \end{subfigure}
  \begin{subfigure}[b]{0.3\linewidth}
    \includegraphics[width=\linewidth]{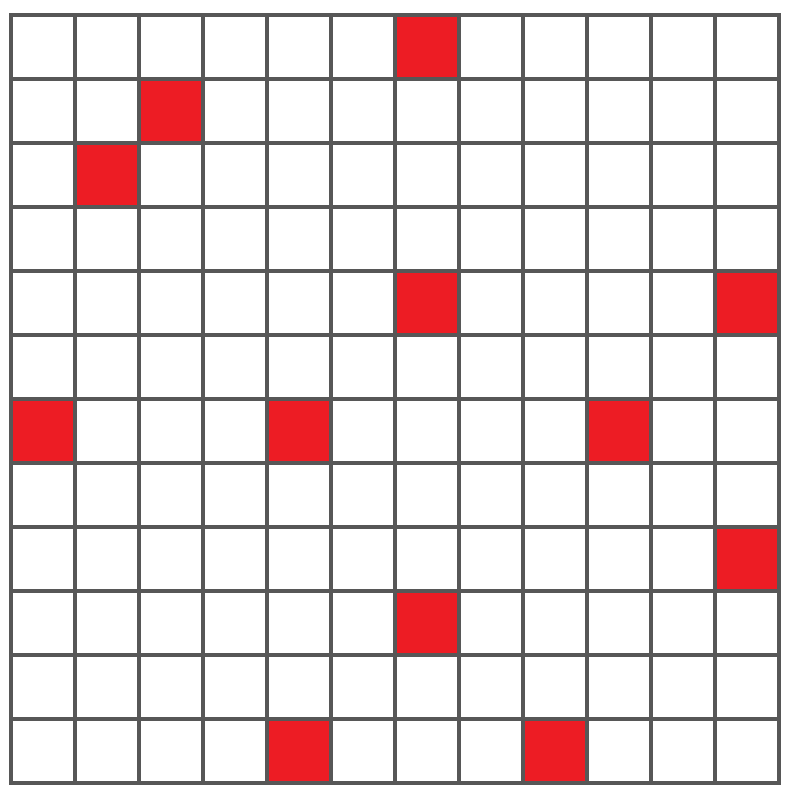}
    \caption{Corresponding $K^s$.}
  \end{subfigure}
  \caption{Illustration of a sparse kernel approximation. Blank entries indicate $0$.}
  \label{fig:app_sparse}
\end{figure}

\section{Numerical examples} \label{sec:numerical}

In this section, we demonstrate numerically the results obtained in the previous sections. We start by demonstrating numerically the error bounds derived in Section~\ref{sec:trunc_pert}. Then, we demonstrate the advantages of the kernel approximation methods proposed in Section~\ref{sec:applications} using both real and synthetic datasets.

The MATLAB code to reproduce the graphs in this section is found in \\ \href{https://github.com/roymitz/perturbation_out_of_sample_extension}{\texttt{github.com/roymitz/perturbation\_kernel\_approximation}}.

\subsection{Perturbation error bounds}

In this section, we demonstrate numerically the behavior of the error bounds~\eqref{eqn:error_trunc_1} and~\eqref{eqn:error_trunc_2} in Propositions~\ref{prop:pert_partial_1} and~\ref{prop:pert_partial_2}, respectively. In our first example, we demonstrate the predicted linear dependence of the errors
$\norm{w_i - \widetilde{w}_i^{(1)}}_2$ and $\norm{w_i - \widetilde{w}_i^{(2)}}_2$ on $\norm{E}_2$. We also show the quadratic dependence of these errors on $\norm{E}_2$ for a proper choice of $\mu$. To that end, we generate a random symmetric matrix $A' \in \mathbb{R}^{1000 \times 1000}$ whose~10 leading eigenvalues are between~1 and~2, and the rest are exactly~$0.5$. We then generate a random symmetric matrix~$E$, and normalize it to have a unit spectral norm. Then, for various values of $c$, we approximate the~10 leading eigenpairs of $A_c = A' + cE$ by the first and second-order approximations~\eqref{eq:pert_expansion} and~\eqref{eq:pert_expansion2} using $\mu = 0$ and $\mu_\text{mean}$. Denote by $v_c$ the leading eigenvector of $A_c$, and by $u^1_c$ and $u^2_c$ its approximations by~\eqref{eq:pert_expansion} and~\eqref{eq:pert_expansion2} using $\mu =0 $, respectively. Denote by $w^1_c$ and $w^2_c$ the respective approximations using $\mu_\text{mean}$. For each~$c$, we measure the errors $\norm{v_c - u^1_c}_2$, $\norm{v_c - u^2_c}_2$, $\norm{v_c - w^1_c}_2$ and $\norm{v_c - w^2_c}_2$. In~\Cref{fig:pert_1}, we plot all the $\text{log}_{10}$-errors versus $\log_{10} \norm{cE}_2$. As predicted by theory, there is a linear dependence between the error in the eigenvector approximation and the norm of the perturbation matrix when using $\mu = 0$. Furthermore, we see that when using $\mu_\text{mean}$, both formulas coincide and give rise to the same second-order approximation, as predicted by Proposition~\ref{prop:order1_is_order2}.


In our second example, we demonstrate the linear and quadratic dependence of the error on $\sum_{j=m+1}^n \abs{\lambda_j -\mu}$, that is, on the unknown eigenvalues of the perturbed matrix. For various values of $c \in \mathbb{R}$, we generate matrices $A'_c  \in \mathbb{R}^{1000 \times 1000}$ as follows. Their 10 leading  eigenvalues are between 1 and 2, and are the same for all values of $c$. The rest of their eigenvalues are exactly $c$. Then, we generate a random symmetric matrix $E \in \mathbb{R}^{1000 \times 1000}$ and normalize it to have a norm of $10^{-6}$. We choose $\norm{E}_2$ to be relatively small, so that its contribution to the error will not mask the effect of $\sum_{j=m+1}^n \abs{\lambda_j -\mu}$. We approximate the 10 leading eigenpairs of $A_c = A'_c + E$ by the first and second-order approximations~\eqref{eq:pert_expansion} and~\eqref{eq:pert_expansion2} using $\mu = 0$, and measure the error in the same way as in the previous example. In~\Cref{fig:pert_2}, we plot  $\log{\norm{v_c - u^1_c}}_2$ and $\log{\norm{v_c - u^2_c}}_2$ versus $\log{\abs{\lambda_j -\mu}} = \log{\abs{\lambda_j}} = \log c$. As predicted by Propositions~\ref{prop:pert_partial_1} and~\ref{prop:pert_partial_2}, there is a linear dependence between the error in the eigenvector approximation and $c$ for the first-order approximation, and a quadratic dependence for the second-order formula. We note that using $\mu_{\text{mean}}$ in this example will cancel out the $\sum_{j=m+1}^n \abs{\lambda_j -\mu}$ term, since in such a case $\mu_{\text{mean}} = c$.

\begin{figure}
  \centering
  \begin{subfigure}[b]{0.4\linewidth} 
    \includegraphics[width=\linewidth]{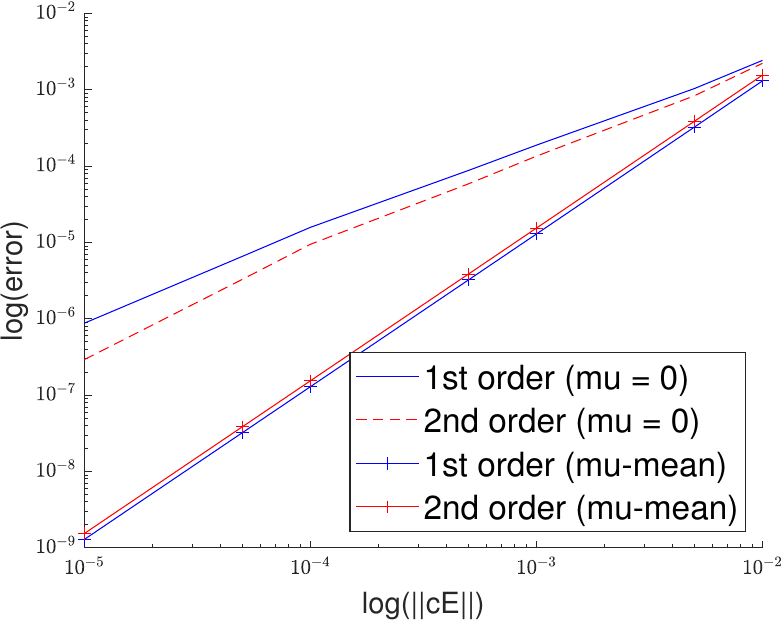}
    \caption{The dependence of the first and second-order formulas on the perturbation matrix norm $\norm{E}_2$. \\ }
    \label{fig:pert_1}
  \end{subfigure} 
  \hspace{1cm}
  \begin{subfigure}[b]{0.4\linewidth}
    \includegraphics[width=\linewidth]{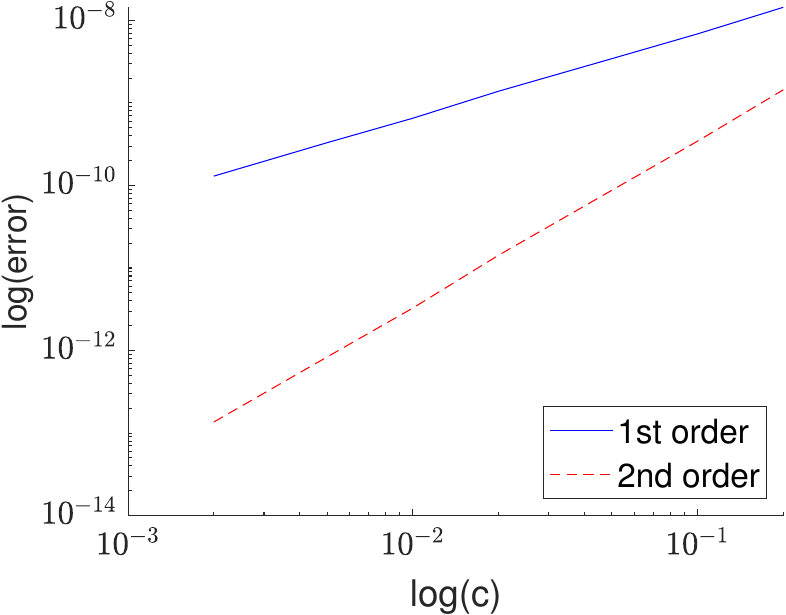}
    \caption{The dependence of the first and second-order formulas on the unknown eigenvalues of the perturbed matrix} $\sum_{j=m+1}^n \abs{\lambda_j -\mu}$.
     \label{fig:pert_2}
  \end{subfigure}
\caption{Numerical demonstration of the error terms in the approximations~\eqref{eq:pert_expansion} and~\eqref{eq:pert_expansion2}. (\subref{fig:pert_1})~$\log{(\text{error})}$ vs. $\log{\norm{cE}_2}$. The slope of both curves corresponding to $\mu = 0$ is~$\approx 1$, demonstrating the linear dependence of the error terms~\eqref{eqn:error_trunc_1} and~\eqref{eqn:error_trunc_2} on $\norm{E}_2$ for both the first and second-order approximations. On the other hand, both curves corresponding to $\mu_{\text{mean}}$ coincide with slope $\approx 2$, demonstrating Proposition~\ref{prop:order1_is_order2}. (\subref{fig:pert_2})~$\log{(\text{error})}$ vs. $\log{\abs{\lambda_j - \mu}}$. The slope of the line corresponding to the first-order approximation is $\approx 1$, whereas the slope of the line corresponding to the second-order approximation is $\approx 2$, demonstrating the linear and quadratic dependence of the first and second-order error terms on $\sum_{j=m+1}^n \abs{\lambda_j -\mu}$, respectively.}
  \label{fig:numerical_pert}
\end{figure}

\subsection{Perturbation-based approximation for synthetic and real data} \label{sec:numeric_2}

In this section, we compare the various kernel approximation methods proposed in Section~\ref{sec:applications} for both synthetic and real data. We demonstrate that the various approximation schemes perform differently, depending on the structure of the kernel matrix. 

As our metric for comparing the performance of the various methods we use the kernel reconstruction error, i.e., if $K_m$ is the best rank-$m$ approximation of the kernel matrix of the entire data (obtained by SVD), and $\widetilde{K}$ is its approximation, we define
\begin{equation*}
    \text{err} = \frac{\norm{K_m - \widetilde{K}}_2}{\norm{K_m}_2}.
\end{equation*}
We note that other metrics we tested performed qualitatively similarly. The metrics we tested were the principal angle~\cite{knyazev2012principal} between the subspace spanned by the kernel's top eigenvectors and the subspace spanned by their approximations, and the subspace projection error $\frac{\norm{UU^T - VV^T}_2}{\norm{UU^T}_2}$, where $U$ and $V$ are the ground truth subspace and its approximation, respectively. The results for all tested metrics are given in Appendix~\ref{app6}.

In all the numerical examples to follow, we use $\mu = 0$ for all approximation schemes. We note that we typically witness a marginal difference between the performance of $\mu = 0 $ and $\mu=\mu_{\text{mean}}$ for real-world data, as such data are usually close to being low-rank. Thus, we do not include the $\mu$-shifted kernel approximation in the experiments of this section. We also set $n = 1000$ in all experiments. We choose $m$ to be the number of components that account for 90\% of the energy of $K$, with a maximum value of 5. For the block-diagonal kernel approximation, we always use two blocks of the same size.

Each experiment is performed as follows. For each kernel type, we generate several kernels of that type, each with a different kernel parameter (to be explained later for each example). Then, we approximate each of the kernels using each of the approximation methods, where in any case the matrix $K^s$ used in the approximation consists of 20\% of the entries of the approximated kernel. For each such kernel approximation, we measure the approximation error versus the Hoyer score~\cite{hurley2009comparing} of the kernel matrix $K$ written as a long vector. The Hoyer score of a vector $v \in \mathbb{R}^n$ is defined by
\begin{equation*}
    \text{Hoyer}(v) = \frac{\sqrt{n} - \frac{\norm{v}_1}{\norm{v}_2}}{\sqrt{n} - 1} .
\end{equation*}
The Hoyer score is a number between~0 and~1, with a higher score corresponding to a sparser vector. We repeat this procedure 20 times for each approximation scheme and kernel parameter, each time with a different random subset of the data. Finally, we plot the mean approximation error for each kernel approximation scheme versus the Hoyer score of the approximated kernel $K$. We also add to the plot the standard deviation of the error of each scheme, presented as a shaded plot around the mean.

\subsubsection{Kernels concentrated along the diagonal} \label{sec:num_p_band}

In this example, we demonstrate the performance of each of the kernel approximation schemes of Section~\ref{sec:applications} using kernel matrices whose energy is concentrated along their diagonal. We build these kernels as follows. We generate a matrix $K$ whose $(i,j)$ and $(j,i)$ entries are $\abs{i - j}^{-\alpha} + X$, where $X$ are i.i.d samples from a normal distribution with mean zero and standard deviation 0.0001. Larger values of $\alpha$ correspond to a more concentrated matrix, whereas smaller values of $\alpha$ correspond to a more spread matrix. 

We execute the experiment described at the beginning of Section~\ref{sec:numeric_2}. The results are presented in Figure~\ref{fig:app_p_band}. We can see that for higher Hoyer scores, which correspond to sparser matrices that are concentrated along the diagonal, the error graphs of the $p$-band and sparse kernel approximation schemes are lower than the error graphs of the other kernel approximation schemes. We conclude that when the kernel is concentrated along its diagonal, the $p$-band kernel approximation and the sparse kernel approximation have superior performance. The performance of the sparse kernel approximation is comparable to the performance of the $p$-band kernel approximation.

\begin{figure}[H]
  \centering
  \captionsetup{justification=centering}
  \begin{subfigure}[b]{0.5\linewidth}
    \includegraphics[width=\linewidth]{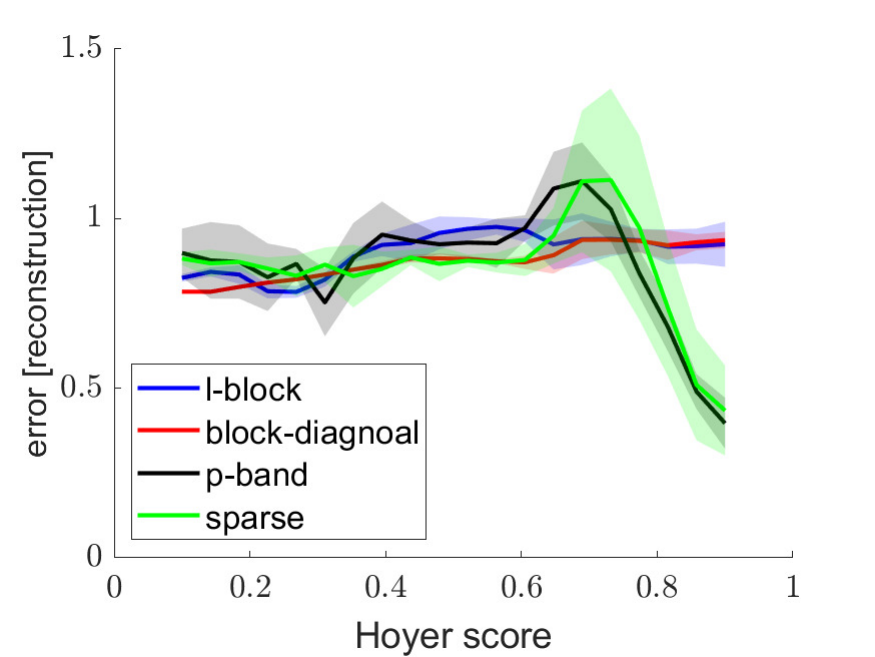}
  \end{subfigure}

  \caption{The error of the various kernel approximation schemes for a kernel concentrated along its diagonal. When the kernel is concentrated along the diagonal, and hence sparser, the $p$-band kernel approximation outperforms the other schemes. The performance of the sparse kernel approximation is comparable to the performance of the $p$-band kernel approximation.}
  \label{fig:app_p_band}
\end{figure}

\subsubsection{Sparse kernels} \label{sec:num_sparse}

In this example, we wish to demonstrate the performance of each kernel approximation scheme using kernel matrices that are sparse. As our kernel matrix, we use the symmetric normalized graph Laplacian matrix, used in Laplacian eigenmaps dimensionality reduction~\cite{belkin2003laplacian}. The $(i,j)$ entry of the graph Laplacian matrix is given by $\exp ( -\norm{x_i - x_j}^2_2 / \sigma )$, followed by some data-dependant normalization.
The reason we use this kernel is that for small values of $\sigma$, this kernel is essentially sparse.

In this subsection, we use real-world datasets taken from the UCI Machine Learning Repository~\cite{Dua:2019}, as described in Table~\ref{tbl:data sets}. For each dataset, we repeat the experiment described in the introduction of Section~\ref{sec:numeric_2} multiple times, each time on a random subset of 1000 points from the dataset, and for several values of~$\sigma$, reflecting the transition between a sparse matrix (small $\sigma$) and a dense matrix (large $\sigma$). The results of this experiment are presented in Figure~\ref{fig:app_mnist}. We can see that for higher Hoyer scores, which correspond to sparser matrices, the error graphs of the sparse kernel approximation scheme are lower than the error graphs of the other kernel approximation schemes. We conclude that when the kernel admits a sparse structure (typically, Hoyer score $>0.75$), the sparse kernel approximation scheme has superior performance. For lower Hoyer scores, however, the kernel is usually no longer sparse and the performance of the sparse kernel approximation scheme is no longer superior to the other kernel approximation schemes. We notice that none of the methods performs well for dense matrices, whereas for sparse matrices, the use of the sparse extension may be the difference between a nearly meaningless result and an informative one.

\begin{figure}[H]
  \centering
  \captionsetup{justification=centering}
  \begin{subfigure}[b]{0.45\linewidth}
    \includegraphics[width=\linewidth]{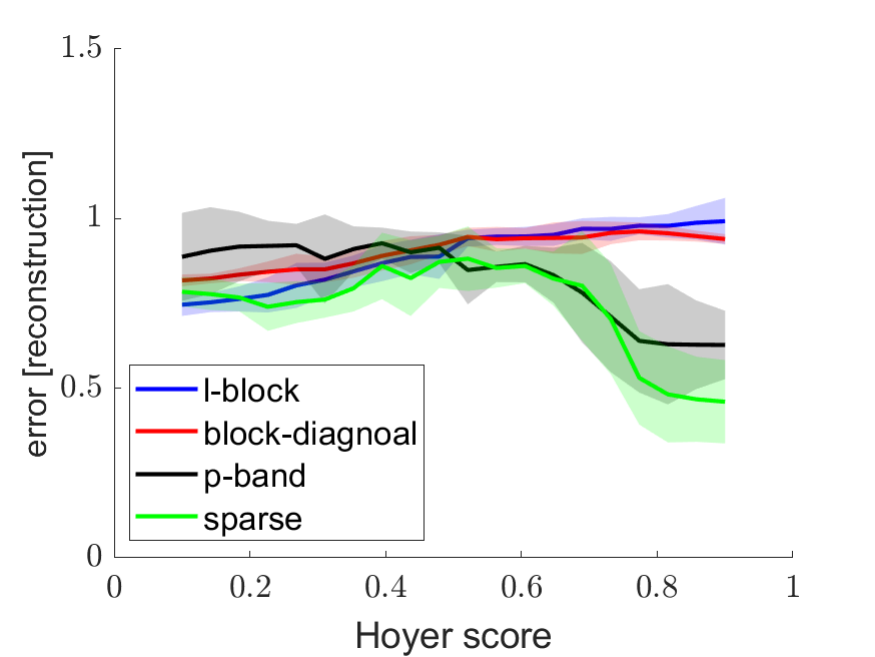}
    \caption{MNIST}
  \end{subfigure}
  \begin{subfigure}[b]{0.45\linewidth}
    \includegraphics[width=\linewidth]{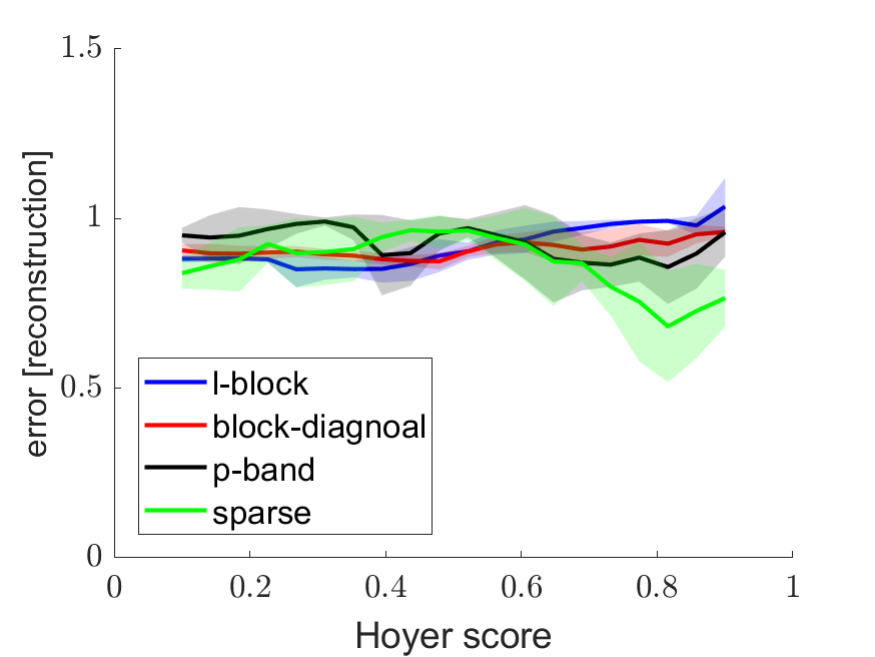}
    \caption{Superconductivity}
  \end{subfigure}
  \\
  \begin{subfigure}[b]{0.45\linewidth}
    \includegraphics[width=\linewidth]{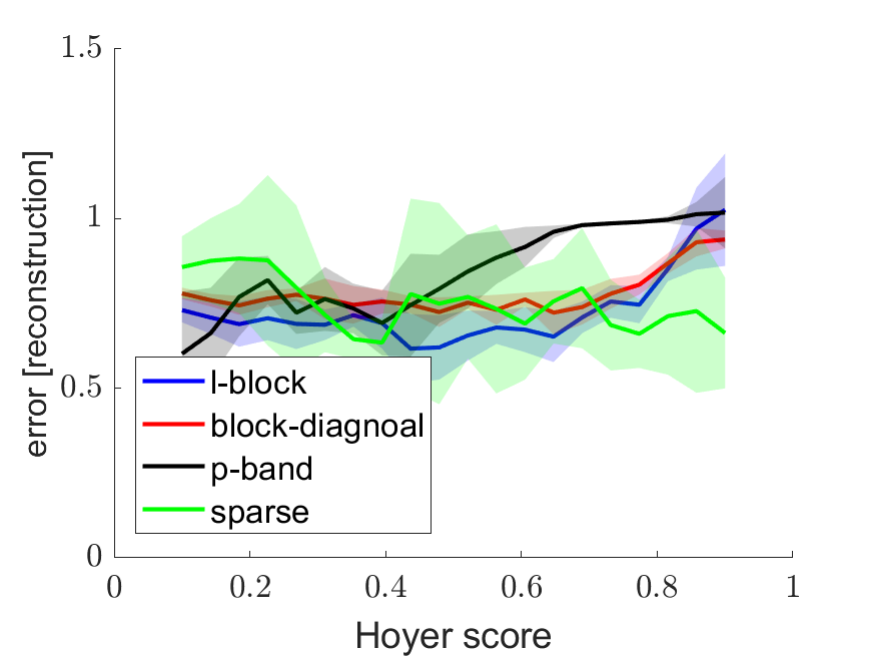}
    \caption{Poker}
  \end{subfigure}
  \begin{subfigure}[b]{0.45\linewidth}
    \includegraphics[width=\linewidth]{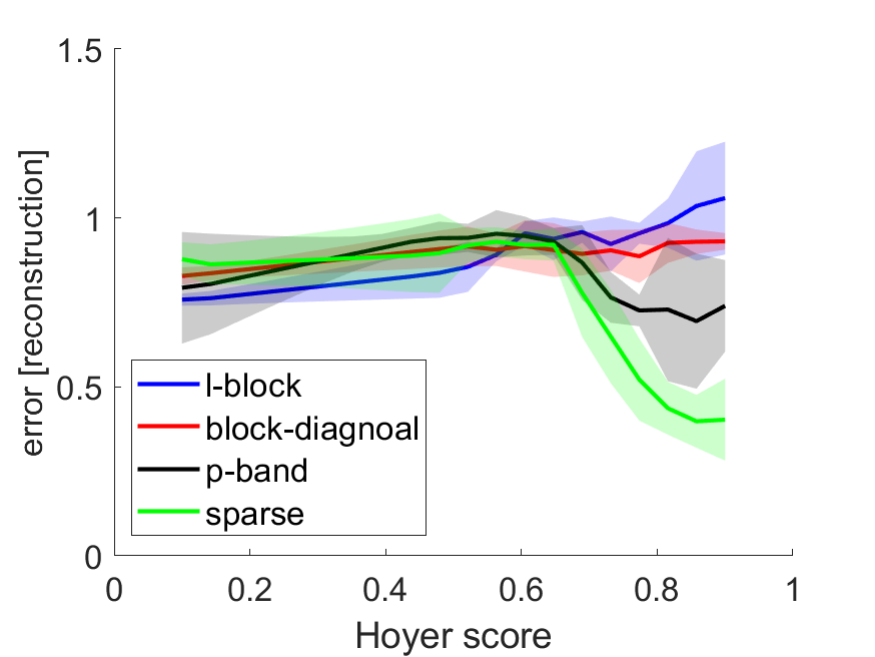}
    \caption{Wine quality}
  \end{subfigure}
  \caption{Error of the graph Laplacian approximation for various datasets and kernel approximation schemes as a function of the approximated kernel Hoyer score. We can see that for sparse kernels (higher Hoyer score, typically $>0.75$), the sparse approximation is superior.}
  \label{fig:app_mnist}
\end{figure}

\begin{table}
\centering
\begin{tabular}{|l|l|p{10cm}|}
\hline
Name              & Dimension & Description                                                                                                                                                           \\ \hline
MNIST             & 784       & Each sample is a grey scale image of a handwritten digit between zero and nine.                                                                                        \\ \hline
Superconductivity & 81        & Each sample contains 81 features extracted from one of 21263 superconductors.                                                                                                \\ \hline
Poker             & 10        & Each sample is a hand consisting of five playing cards drawn from a standard deck of 52 cards. Each card is described using two attributes (suit and rank).           \\ \hline
Wine quality      & 11        & Each sample corresponds to a variant of a Portuguese wine, where the 11 attributes are numerical characteristics of the wine such as acidity, pH, residual sugar etc. \\ \hline
\end{tabular}
\caption{Real-world datasets used.}
\label{tbl:data sets}
\end{table} 

\section{Summary and future work} \label{sec:summary}

In this paper, we propose a kernel approximation framework that is based on perturbation theory. We prove that this framework is a generalization of the popular Nystr{\"o}m method and some of its variants. Furthermore, contrary to existing error bounds for the Nystr{\"o}m method, our framework provides error bounds for the individual eigenvectors. This is useful when the approximation is used as part of a dimensionality reduction procedure. Our kernel approximation framework is very flexible, and can thus take advantage of the structure of the kernel matrix. We demonstrate our theoretical derivations numerically for kernel matrices that are either sparse or concentrated along their diagonal. Currently, our framework does not handle kernels with degenerate eigenvalues, nor crossovers between eigenvalues. The latter might cause a ``change of order" of the eigenvectors.

For a future work, one might consider schemes to construct kernel matrices with structure that can take advantage of our framework. For example, the MEKA algorithm orders the data in clusters, resulting in a block-diagonal kernel.


\newpage

\bibliography{Nys_bib}
\bibliographystyle{plain}

\appendix

\section{Proof of Proposition~\ref{prop:pert_partial_1}} \label{app1}

Let $1 \leq i \leq m$. Ignoring the $O(\norm{E}^2_2)$ term, we split~\eqref{eqn:pert_org_vecs} into the known and unknown terms, resulting in
\begin{equation} \label{eq:formula_split}
\widetilde{w}_{i} = v_{i} + \sum_{k=1, k \neq i}^{m} \frac{ \langle Ev_i,v_k \rangle }{t_i-t_k}v_k + \sum_{k={m+1}}^n \frac{ \langle Ev_i,v_k \rangle }{t_i-t_k}v_k .
\end{equation}
As the second term in~\eqref{eq:formula_split} is unknown, we approximate it by replacing the unknown eigenvalues with a parameter $\mu$, resulting in
\begin{equation} \label{eq:formula_split2}
\sum_{k={m+1}}^n \frac{ \langle Ev_i,v_k \rangle }{t_i-\mu}v_k = \frac{1}{t_i-\mu}\sum_{k={m+1}}^n  \langle Ev_i,v_k \rangle v_k = \frac{1}{t_i-\mu} (Ev_i - V^{(m)}V^{(m)T}Ev_i) =  \frac{1}{t_i-\mu} r_i,
\end{equation}
where $r_{i}$ is defined in~\eqref{eqn:r}.
Formula~\eqref{eq:pert_expansion} follows by replacing the second term in~\eqref{eq:formula_split} with the rightmost term in~\eqref{eq:formula_split2}. We denote the approximation error introduced into the approximation~\eqref{eq:pert_expansion} by $e_{i}$, that is
\begin{equation*}
    e_i = \norm{\sum_{k={m+1}}^n \frac{ \langle Ev_i,v_k \rangle }{t_i-t_k}v_k -  \frac{1}{t_i-\mu} r_i}_2.
\end{equation*}
By using the identity
\begin{equation*} 
\frac{1}{t_i - t_k} = \frac{1}{t_i - \mu} + \frac{t_k - \mu}{(t_i - t_k)(t_i - \mu)},
\end{equation*}
we get
\begin{equation*} 
e_i =  \norm{\sum_{k={m+1}}^n \frac{t_k - \mu}{(t_i - t_k)(t_i - \mu)} \langle Ev_i,v_k \rangle v_k }_2.
\end{equation*}
By the triangle inequality and the Cauchy-Schwarz inequality we get
\begin{equation*} 
e_i \leq \sum_{k={m+1}}^n \frac{\abs{t_k - \mu}}{\abs{t_i - t_k}\abs{t_i - \mu}}\abs{ \langle Ev_i,v_k \rangle } \leq \frac{\norm{E}_2}{\abs{t_i - t_{m+1}}\abs{t_i - \mu}}\sum_{k={m+1}}^n \abs{t_k - \mu} .
\end{equation*}
Recalling that the original perturbation approximation~\eqref{eqn:pert_org_vecs} induces an error of $O(\norm{E}^2_2$) concludes the proof.

\section{Proof of Proposition~\ref{prop:pert_partial_2}} \label{app2}

Let $1 \leq i \leq m$. Ignoring the $O(\norm{E}^2_2)$ term, we split~\eqref{eqn:pert_org_vecs} into the known and unknown terms, resulting in
\begin{equation} \label{eq:formula_split3}
\widetilde{w}_{i} = v_{i} + \sum_{k=1, k \neq i}^{m} \frac{ \langle Ev_i,v_k \rangle }{t_i-t_k}v_k + \sum_{k={m+1}}^n \frac{ \langle Ev_i,v_k \rangle }{t_i-t_k}v_k .
\end{equation}
To obtain~\eqref{eq:pert_expansion2}, we expand the unknown (second) term in~\eqref{eq:formula_split3} by using the identity 
\begin{equation} \label{eq:id2}
\frac{1}{t_i - t_k} = \frac{1}{t_i - \mu} + \frac{t_k - \mu}{(t_i - \mu)^2} + \frac{(t_k - \mu)^2}{(t_i - t_k)(t_i - \mu)^2}.
\end{equation}
We note that
\begin{align}
\sum_{k={m+1}}^n  \langle Ev_i,v_k \rangle (t_k - \mu)v_k &= \sum_{k={m+1}}^n  \langle Ev_i,v_k \rangle t_kv_k - \mu \sum_{k={m+1}}^n  \langle Ev_i,v_k \rangle v_k \\
&= \sum_{k={m+1}}^n  \langle Ev_i,v_k \rangle A'v_k - \mu \sum_{k={m+1}}^n  \langle Ev_i,v_k \rangle v_k  \\
&= A'r_i - \mu r_i, \label{eq:err_bound}
\end{align} 
where $r_{i}$ is defined in~\eqref{eqn:r}. Using~\eqref{eq:err_bound} and~\eqref{eq:formula_split2}, we get that the unknown term in~\eqref{eq:formula_split3} can be written as
\begin{equation} \label{eq:formula_split4}
 \sum_{k={m+1}}^n \frac{ \langle Ev_i,v_k \rangle }{t_i-t_k}v_k =  \frac{1}{t_i- \mu} r_i -  \frac{\mu}{(t_i- \mu)^2} r_i +  \frac{1}{(t_i- \mu)^2} A'r_i + \sum_{k={m+1}}^n \frac{(t_k - \mu)^2}{(t_i - t_k)(t_i - \mu)^2 }  \langle Ev_i,v_k \rangle  v_k .
\end{equation}
The second-order formula~\eqref{eq:pert_expansion2} now follows by discarding the last term in~\eqref{eq:formula_split4}. Denoting the approximation error of equation~\eqref{eq:pert_expansion2} by $e_{i}$, we have
\begin{equation*}
    e_i = \norm{\sum_{k={m+1}}^n \frac{ \langle Ev_i,v_k \rangle }{t_i-t_k}v_k -  \bigg( \frac{1}{t_i-\mu} r_i -  \frac{\mu}{(t_i- \mu)^2} r_i +  \frac{1}{(t_i- \mu)^2} A'r_i \bigg) }_2 .
\end{equation*}
Using~\eqref{eq:id2} and the triangle and the Cauchy-Schwarz inequalities, we obtain
\begin{align*} 
e_i &=
 \norm{\sum_{k={m+1}}^n \frac{(t_k - \mu)^2}{(t_i - t_k)(t_i - \mu)^2} \langle Ev_i,v_k \rangle v_k  }_2 \\
&\leq  \sum_{k={m+1}}^n \frac{\abs{t_k - \mu}^2}{\abs{t_i - t_k}\abs{t_i - \mu}^2}\abs{ \langle Ev_i,v_k \rangle } \\
&\leq \frac{\norm{E}_2}{\abs{t_i - t_{m+1}}\abs{t_i - \mu}^2}\sum_{k={m+1}}^n \abs{t_k - \mu}^2 .
\end{align*}
Recalling that the original perturbation approximation~\eqref{eqn:pert_org_vecs} induces an error of $O(\norm{E}^2_2$) concludes the proof.

\section{Runtime and space complexity} \label{app:runtime}

In this section, we discuss the runtime and space complexities of formulas~\eqref{eq:pert_expansion} and~\eqref{eq:pert_expansion2}. We start with the first-order formula~\eqref{eq:pert_expansion}. The space complexity of the first-order formula is $O(mn)$, since it needs to store in memory the~$m$ leading eignevectors of $A'$. As of the runtime of the first-order formula, the computation of $r_i$ of~\eqref{eqn:r} involves the calculation of $Ev_i$ that requires $O(\text{nnz}(E))$ operations. The result is then multiplied by $V^{(m)T}$, which requires $O(mn)$ operations, and then by $V^{(m)}$, which also requires $O(mn)$ operations. Thus, the total runtime complexity for computing all $\{r_i\}_{i=1}^{m}$ is $O(m \cdot \text{nnz}(E) + m^2n)$ operations. The first-order formula also requires to compute $O(m^2)$ terms of the form $ \langle Ev_i,v_k \rangle v_k$ for $1 \leq i \neq k \leq m$. Each such term requires $O(\text{nnz}(E) + n)$ operations, resulting in a total of $O(m^2 \cdot \text{nnz}(E) + m^2n)$ operations for all eigenvectors. We conclude that evaluating the first-order formula requires a total of $O(m^2 \cdot \text{nnz}(E) + m^2n)$ operations.

The analysis of the second-order formula~\eqref{eq:pert_expansion2} is similar, except that it requires also to store~$A'$, which requires $O( \text{nnz}(A'))$ memory, and to compute $A'r_i$, which requires additional $O(m \cdot  \text{nnz}(A'))$ operations.  We conclude that evaluating the second-order formula requires a total of $O(m^2 \cdot \text{nnz}(E) + m \cdot  \text{nnz}(A') + m^2n)$ operations.

\section{Proof of Proposition~\ref{prop:order1_is_order2}} \label{app3}

Let $i \in \left \{1,\ldots,m \right \}$. Since $\mu = \delta$, we get that if $A'r_i = \delta r_i$ then the last two terms in~\eqref{eq:pert_expansion2} cancel out and~\eqref{eq:pert_expansion2} reduces to the first-order formula~\eqref{eq:pert_expansion}. Thus, in order to prove that $w_{i}^{(1)} = w_{i}^{(2)}$, it is sufficient to prove that under the settings of the proposition $A'r_i = \delta r_i$. Indeed,
\begin{align*}
    A'r_i &= A'  \Big( I - V^{(m)}V^{(m)T} \Big) Ev_i =  \Big( V^{(m)}TV^{(m)T} + \delta I \Big) \Big( I - V^{(m)}V^{(m)T} \Big) Ev_i \\
    &= \Big(  V^{(m)}TV^{(m)T} - V^{(m)}TV^{(m)T} + \delta I - \delta V^{(m)}V^{(m)T} \Big) Ev_i \\
    &= \delta \Big(  I -  V^{(m)}V^{(m)T} \Big) Ev_i  = \delta r_i.
\end{align*}

For the error, we note that the $n - m$ unknown eigenvalues of a matrix $A'$ of the form $A' = V^{(m)}TV^{(m)T} + \delta I$ are exactly $\delta$, and thus, when choosing $\mu = \delta$, the first term in~\eqref{eqn:error_trunc_1} and~\eqref{eqn:error_trunc_2} cancels out and we are left with only the $O(\norm{E}^2_2)$ term.

\section{Proof of Proposition~\ref{prop:nys_is_pert}} \label{app4}

Let $i \in \left \{1,\ldots,m \right \}$. By~\eqref{eqn:pert_extension}, for $\mu=0$ 
\begin{equation} \label{eqn:pert_nys0_proof}
\widetilde{u}_i = u_{i}^s + \sum_{k=1, k \neq i}^{m} \frac{((K - K^s)u_i^s,u_k^s)}{\lambda_i^s-\lambda_k^s}u_k^s  +  \frac{1}{\lambda_i^s} \big( I - U^{s(m)}U^{s(m)^T} \big)(K - K^s)u^s_i,
\end{equation}
and by~\eqref{eqn:pert_ext_vals},
\begin{equation} \label{eqn:pert_ext_vals2}
\widetilde{\lambda}_i = \lambda^s_i + u_i^{sT}(K - K^s)u^s_i.
\end{equation}

We first prove that $ \hat{u}_i = \sqrt{\frac{m}{n}} \widetilde{u}_i$ (see~\eqref{eqn:ev equivalence}). We start by simplifying~\eqref{eqn:pert_nys0_proof} based on the specific choice of~$K^s$. We make the following observations. First, we note that for all $i=1,\ldots,m$,  the last $n-m$ entries of $u_i^s$ are $0$, and the top left $m \times m$ submatrix of $K - K^s$ is $0$. This implies that the first $m$ entries of $(K - K^s)u_i^s$ are $0$, and so $\langle (K - K^s)u_i^s , u_k^s \rangle = 0$ for all $1 \leq i,k \leq m$. A direct consequence of the latter is that $U^{s(m)}U^{s(m)T}(K - K^s)u_i^s = 0$ for all $1 \leq i \leq m$. Thus,~\eqref{eqn:pert_nys0_proof} reduces to
\begin{equation} \label{eq:nystrom_pert_reduced}
\widetilde{u}_i = u_{i}^s  +  \frac{1}{\lambda_i^s}(K - K^s)u_i^s.
\end{equation}
Next, we note that the first term in~\eqref{eq:nystrom_pert_reduced}, $u_i^s$ , is non-zero only on its first $m$ entries, whereas the second term, $\frac{1}{\lambda_i^s}(K - K^s)u_i^s$ is non-zero only on its last $n-m$ entries. This means that the first $m$ entries of $\widetilde{u}_i$ are exactly the first $m$ entries $u_{i}^s$, and the last $n-m$ entries of $\widetilde{u}_i$ are equal to the last $n-m$ entries of $\frac{1}{\lambda_i^s}(K - K^s)u_i^s$. We start by proving the equivalence between the first~$m$ entries of $\widetilde{u}_i$ and $\hat{u}_i$. Let $1 \leq p \leq m$. Denote by $A_{p\rightarrow}$ the $p$'th row of a matrix $A$. Denote by $C$ the $n \times m$ matrix consisting of the first $m$ columns of $K$. We note that for the Nystr{\"o}m method (see~\eqref{eq:nystrom})
\begin{equation*}
\hat{u}_{i,p} = \sqrt{\frac{m}{n}} \frac{1}{\lambda_i'} C_{p \rightarrow} u_i' = \sqrt{\frac{m}{n}} \frac{1}{\lambda_i'} K'_{p \rightarrow} u_i' = \sqrt{\frac{m}{n}} \frac{1}{\lambda_i'} \lambda_i' u_{i,p}' = \sqrt{\frac{m}{n}} u_{i,p}'.
\end{equation*}
Thus, the first $m$ entries of the approximated vector in the Nystr{\"o}m method are merely a re-scaling of the vector $u_i'$ by $\sqrt{\frac{m}{n}}$. Since $K^s$ is equal to $K'$ padded with zeros, we get that the first $m$ entries of $u_i^s$ are exactly $u_i'$. Thus, we conclude that the first $m$ entries of $\hat{u}_i$ and  $\sqrt{\frac{m}{n}} \widetilde{u}_i$ are identical.

Next, we prove the equivalence between the last $n - m$ entries of $\widetilde{u}_i$ and $\hat{u}_i$. Let $m + 1 \leq p \leq n$. We have for the Nystr{\"o}m method (see~\eqref{eq:nystrom})
\begin{equation} \label{eq:part_of_proof1}
\hat{u}_{i,p} = \sqrt{\frac{m}{n}} \frac{1}{\lambda_i'} C_{p \rightarrow} u_i' ,
\end{equation}
and for the perturbation approximation, by using~\eqref{eq:nystrom_pert_reduced}, and since the last $n-m$ entries of $u_{i}^{s}$ are~0,
\begin{equation} \label{eq:part_of_proof2}
\widetilde{u}_{i,p} = \frac{1}{\lambda_i^s}(K - K^s)_{p \rightarrow} u_i^s = \frac{1}{\lambda_i^s}(K_{p \rightarrow} u_i^s - 0) = \frac{1}{\lambda_i^s}C_{p \rightarrow} u_i',
\end{equation}
where the last equality follows since (as explained above) the first $m$ entries of $u_i^s$ are exactly~$u_i'$. Finally, since $\lambda_i' = \lambda_i^s$ for $1 \leq i \leq m$, we conclude by~\eqref{eq:part_of_proof1} and~\eqref{eq:part_of_proof2} that $\hat{u}_{i,p} = \sqrt{\frac{m}{n}} \widetilde{u}_{i,p}$ for $m+1 \leq p \leq n$, that is, the last $n - m$ entries of $\hat{u}_i$ and  $\sqrt{\frac{m}{n}} \widetilde{u}_i$ are also identical.

We now prove the equivalence of the eigenvalues. By the same arguments as above, we note that for all $1 \leq i \leq m$, $u_i^{sT}(K - K^s)u^s_i = 0$. Thus, by~\eqref{eqn:pert_ext_vals2} we have $\widetilde{\lambda}_i = \lambda_i^s $ and consequently, using~\eqref{eq:nystrom_vals},
\begin{equation*}
    \hat{\lambda}_i = \frac{n}{m} \lambda_i' = \frac{n}{m} \lambda_i^s = \frac{n}{m} \widetilde{\lambda}_i ,
\end{equation*}
as required. 

\section{Proof of Proposition~\ref{prop:ss_nys_is_pert}} \label{app5}

Denote by $K_{\text{shift}}^s \in \mathbb{R}^{n \times n}$ the top left $m \times m $ submatrix of $K_{\text{shift}}$ (see~\eqref{eq:shift}) padded with zeros. By the equivalence of the Nystr{\"o}m method and the perturbation approximation proved in Proposition~\ref{prop:nys_is_pert}, the spectral shifted Nystr{\"o}m method is equivalent to the perturbation approximation of $K_{\text{shift}}^s$  using $\mu=0$. Thus, it suffices to prove that the perturbation approximation of the eigenpairs of $K_{\text{shift}}^s$ to the eigenpairs of $K_{\text{shift}}$ with $\mu = 0$ equals to the perturbation approximation of the eigenpairs of $K^s$ to the eigenpairs of~$K$ with $\mu = \delta$.

Let the top $m$ eigenpairs of $K^s$ be $\{(\lambda_i^s, u_i^s))\}_{i=1}^{m}$. It follows that the top $m$ eigenpairs of $K^s_{\text{shift}}$ are~$\{(\lambda_i^s - \delta, u_i^s))\}_{i=1}^{m}$. Let $1 \leq i \leq m$. By~\eqref{eq:nystrom_pert_reduced}, the perturbation approximation~\eqref{eqn:pert_extension} of the eigenvectors of $K_{\text{shift}}^s$ to eigenvectors of $K_{\text{shift}}$ with $\mu = 0$ reduces to
\begin{equation*}
\widetilde{u}_i = u_{i}^s  +  \frac{1}{\lambda_i -\delta - 0}(K_{\text{shift}} - K_{\text{shift}}^s)u_i^s, 
\end{equation*}
whereas the perturbation approximation of the eigenvectors of $K^s$ to eigenvectors of $K$ with $\mu = \delta$ reduces to
\begin{equation*}
\hat{u}_i = u_{i}^s  +  \frac{1}{\lambda_i -\delta}(K - K^s)u^s_i.
\end{equation*}
But since the last $n-m$ entries of $u_i^s$ are 0, and the entries of $K_{\text{shift}} - K_{\text{shift}}^s$ are equal to those of $K - K^s$ except for the last $n-m$ diagonal elements, we have that
\begin{equation} \label{eq:eq_same}
(K_{\text{shift}} - K_{\text{shift}}^s)u_i^s = (K - K^s)u_i^s ,
\end{equation}
and we conclude that $\widetilde{u}_i = \hat{u}_i$.

For the eigenvalues, the perturbation approximation of the eigenvalues of $K_{\text{shift}}^s$ to the eigenvalues of $K_{\text{shift}}$ (see~\eqref{eqn:pert_ext_vals}) yields 
\begin{equation*}
\widetilde{\lambda}_i = (\lambda_i^s - \delta) + u_i^s (K_{\text{shift}} - K_{\text{shift}}^s)u_i^s .    
\end{equation*}
We note that by~\eqref{eqn:pert_ext_vals},~\eqref{eq:eq_same} and by the proof of this section for the eigenvectors, if $\{ (\tau_i, v_i) \}_{i=1}^m$ are the approximated eigenpairs of a matrix $A$, then $\{ (\tau_i + \delta, v_i) \}_{i=1}^m$ are the approximated eigenpairs of the matrix $A_{\text{shift}} = A + \delta I$. Thus, in order to recover the approximation of the eigenvalues of $K$, we need to shift the approximated eigenvalues $\{ \widetilde{\lambda}_i \}_{i=1}^m$ back by $\delta$, yielding

\begin{equation*}
    \widetilde{\lambda}_i' = \lambda_i^s + u_i^s (K_{\text{shift}} - K_{\text{shift}}^s)u_i^s.
\end{equation*}
On the other hand, the perturbation approximation of $K^s$ yields
\begin{equation*}
    \hat{\lambda}_i = \lambda_i^s + u_i^s (K - K^s)u_i^s.
\end{equation*}
By~\eqref{eq:eq_same}, we have that $\widetilde{\lambda}_i' = \hat{\lambda}_i$.

\section{Performance of various error metrics} \label{app6}

In this section, we provide the results of Section~\ref{sec:num_sparse} for the MNIST and wine datasets for three different error measures, demonstrating that the performance of our approximation scheme is qualitatively similar in various reasonable error metrics. The results are presented in Figure~\ref{fig:all_metrics_ok_1} and in Figure~\ref{fig:all_metrics_ok_2} for the MNIST and wine datasets respectively. 

\begin{figure}[H]
  \centering
  \captionsetup{justification=centering}
  \begin{subfigure}[b]{0.3\linewidth}
    \includegraphics[width=\linewidth]{mnist_10_pct_recon.pdf}
    \caption{Reconstruction error}
  \end{subfigure}
  \begin{subfigure}[b]{0.3\linewidth}
    \includegraphics[width=\linewidth]{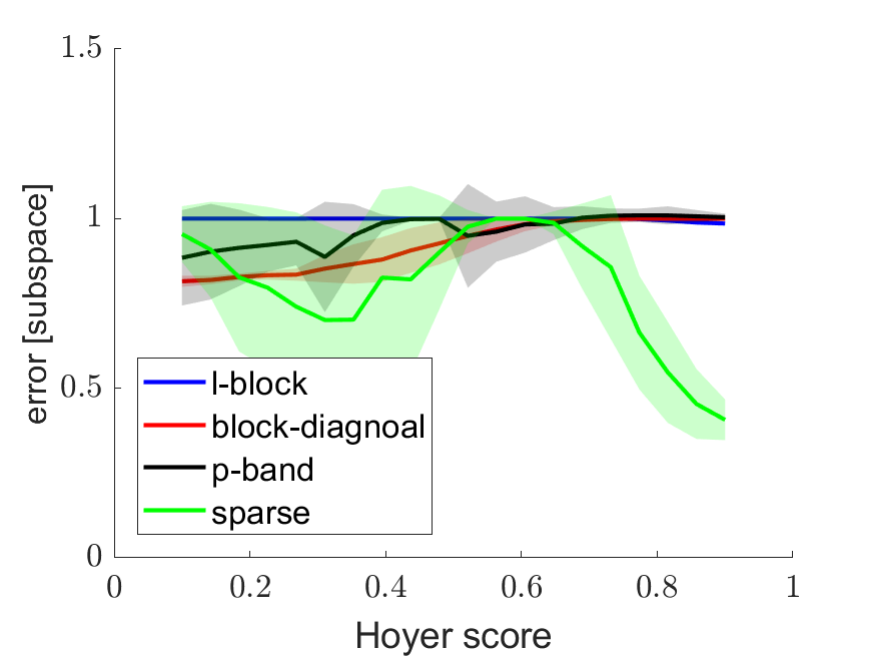}
    \caption{Subspace projection error}
  \end{subfigure}
  \begin{subfigure}[b]{0.3\linewidth}
    \includegraphics[width=\linewidth]{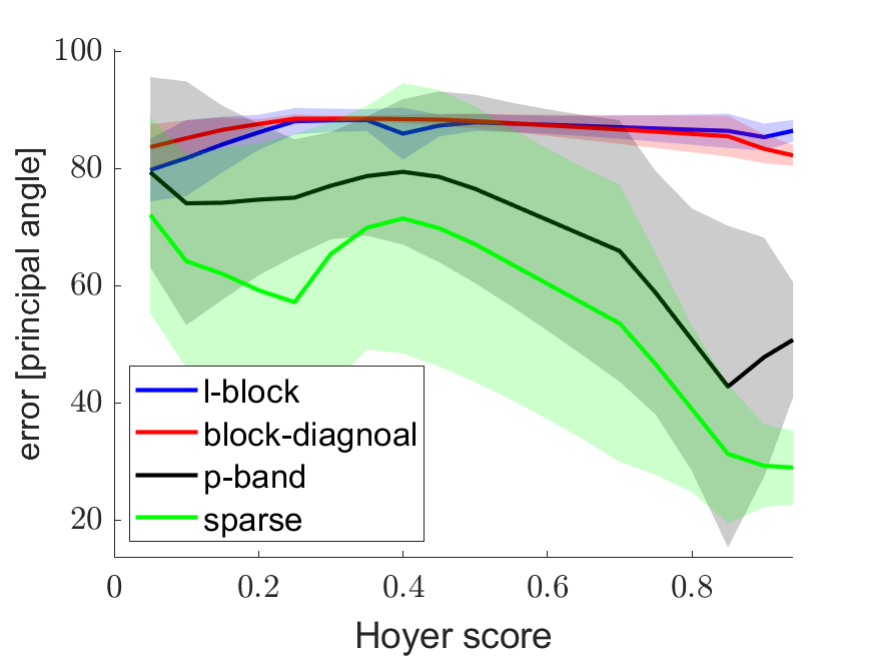}
    \caption{Principle angle}
  \end{subfigure}
  \caption{Approximation error for the MNIST dataset in various error metrics.}
  \label{fig:all_metrics_ok_1}
\end{figure}

\begin{figure}[H]
  \centering
  \captionsetup{justification=centering}
  \begin{subfigure}[b]{0.3\linewidth}
    \includegraphics[width=\linewidth]{wine_10_pct_recon.pdf}
    \caption{Reconstruction error}
  \end{subfigure}
  \begin{subfigure}[b]{0.3\linewidth}
    \includegraphics[width=\linewidth]{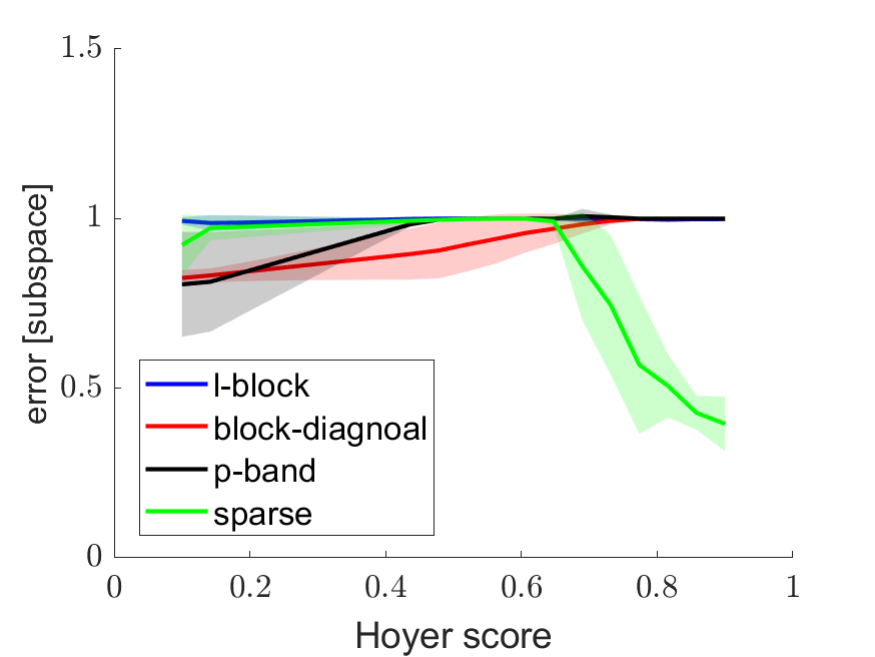}
    \caption{Subspace projection error}
  \end{subfigure}
  \begin{subfigure}[b]{0.3\linewidth}
    \includegraphics[width=\linewidth]{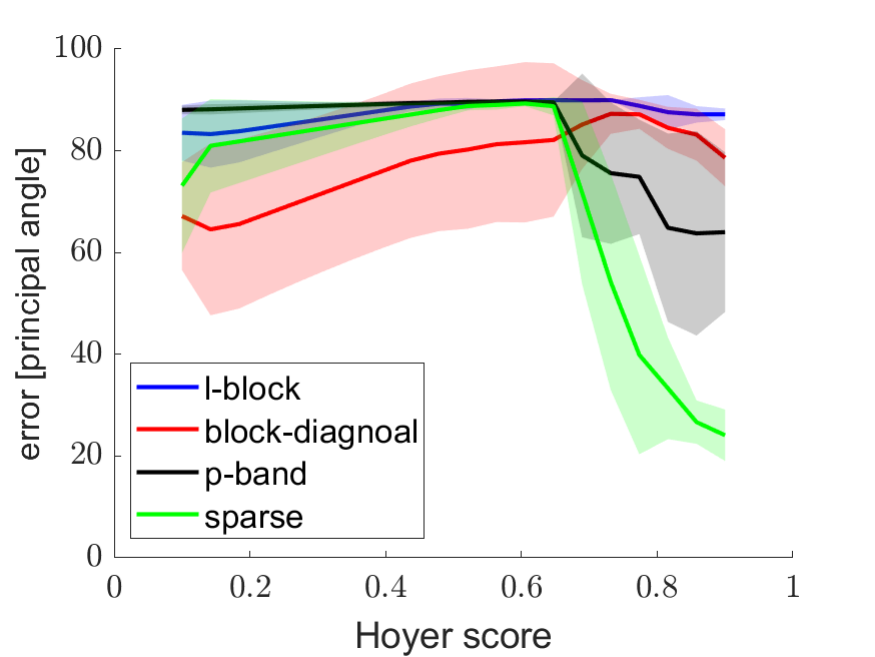}
    \caption{Principle angle}
  \end{subfigure}
  \caption{Approximation error for the wine dataset in various error metrics.}
  \label{fig:all_metrics_ok_2}
\end{figure}

\vskip 0.2in

\end{document}